\pdfoutput=1

\documentclass[11pt]{article}

\usepackage[table,dvipsnames]{xcolor}
\usepackage[final]{acl}

\usepackage{times}
\usepackage{latexsym}

\usepackage[T1]{fontenc}

\usepackage[utf8]{inputenc}

\usepackage{microtype}

\usepackage{inconsolata}

\usepackage{graphicx}
\usepackage{paralist}
\usepackage{algorithm}
\usepackage{algorithmic}
\usepackage{amsmath}
\usepackage{multirow}
\usepackage{booktabs}
\usepackage{colortbl}
\usepackage{amssymb}
\usepackage[normalem]{ulem}
\definecolor{myGreen}{RGB}{76, 153, 0}    
\definecolor{myRed}{RGB}{165, 42, 42}
\usepackage{enumitem}

\usepackage{amsmath, amsthm, amsfonts}

\theoremstyle{plain}
\newtheorem{theorem}{Theorem}[section] 
\newtheorem{lemma}[theorem]{Lemma} 
\newtheorem{proposition}[theorem]{Proposition}
\newtheorem{corollary}[theorem]{Corollary}

\theoremstyle{definition}

\theoremstyle{remark}

\usepackage[most]{tcolorbox}

\newcommand{\fakeparagraph}[1]{\noindent\textbf{#1.}}

\theoremstyle{definition} 
\newtheorem*{takeaway}{Takeaway}
\tcolorboxenvironment{takeaway}{%
  enhanced,breakable,
  colback=violet!3,colframe=violet!55,
  boxrule=0.6pt,arc=2pt,
  left=0pt,right=0pt,top=0pt,bottom=0pt,
}

\tcolorboxenvironment{lemma}{%
  enhanced,breakable,
  colback=cyan!6,colframe=cyan!55,
  boxrule=0.6pt,arc=2pt,
  left=0pt,right=0pt,top=0pt,bottom=0pt,
}
\tcolorboxenvironment{proposition}{%
  enhanced,breakable,
  colback=cyan!6,colframe=cyan!55,
  boxrule=0.6pt,arc=2pt,
  left=0pt,right=0pt,top=0pt,bottom=0pt,
}
\tcolorboxenvironment{theorem}{%
  enhanced,breakable,
  colback=cyan!6,colframe=cyan!55,
  boxrule=0.6pt,arc=2pt,
  left=0pt,right=0pt,top=0pt,bottom=0pt,
}
\tcolorboxenvironment{corollary}{%
  enhanced,breakable,
  colback=cyan!6,colframe=cyan!55,
  boxrule=0.6pt,arc=2pt,
  left=0pt,right=0pt,top=0pt,bottom=0pt,
}

\usepackage{listings}

\definecolor{PromptBar}{RGB}{160,205,155}      
\definecolor{PromptFrame}{RGB}{170,210,165}    
\definecolor{PromptBg}{RGB}{245,252,244}       

\tcbset{
  aclprompt/.style={
    enhanced,
    colback=PromptBg,
    colframe=PromptFrame,
    boxrule=0.8pt,
    arc=2mm,
    left=4mm,right=4mm,top=3mm,bottom=3mm,
    colbacktitle=PromptBar,
    coltitle=black,          
    fonttitle=\bfseries,
    attach boxed title to top left={xshift=2mm,yshift=-1mm},
    boxed title style={arc=2mm,boxrule=0pt},
  }
}

\lstdefinestyle{promptstyle}{
  basicstyle=\ttfamily\footnotesize,
  columns=fullflexible,
  breaklines=true,
  breakatwhitespace=false,
  showstringspaces=false,
  frame=none,
  escapeinside={(*@}{@*)}
}
\title{%
  \makebox[\textwidth][c]{%
    \raisebox{-0.5\height}{%
      \includegraphics[height=3.8em]
        {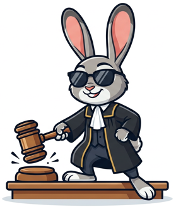}%
    }%
    \hspace{0.2em}%
    \raisebox{-0.5\height}{%
      \shortstack[c]{%
        JuDi: Revisiting Judge Decoding from First Principles\\
        via Training-Free Distributional Divergence
      }%
    }%
  }%
}

\author{
    Shengyin Sun\textsuperscript{\rm 1}, 
    Yiming Li\textsuperscript{2$\dagger$}, 
    Renxi Liu\textsuperscript{\rm 2}, 
    Weizhe Lin\textsuperscript{\rm 2}, \\ 
    \textbf{Hui-Ling Zhen}\textsuperscript{\rm 2}, 
    \textbf{Xianzhi Yu}\textsuperscript{\rm 2}, 
    \textbf{Mingxuan Yuan}\textsuperscript{\rm 2}, 
    \textbf{Chen Ma}\textsuperscript{1$\dagger$}\\
    \textsuperscript{\rm 1}City University of Hong Kong
    \textsuperscript{\rm 2}Huawei Technologies \\
    \texttt{shengysun4-c@my.cityu.edu.hk, li.yiming3@huawei.com, chenma@cityu.edu.hk}
}

\begin{document}
\maketitle

\begingroup
\renewcommand{\thefootnote}{\fnsymbol{footnote}}
\footnotetext[2]{Corresponding authors.}
\endgroup

\begin{abstract}
Judge Decoding accelerates LLM inference by relaxing the strict verification of Speculative Decoding, yet it typically relies on expensive and noisy supervision. In this work, we revisit this paradigm from first principles, revealing that the ``criticality'' scores learned via costly supervision are intrinsically encoded in the draft-target distributional divergence. We theoretically prove a structural correspondence between learned linear judges and Kullback-Leibler (KL) divergence, demonstrating they rely on the same underlying logit primitives. Guided by this, we propose a simple, training-free verification mechanism based on KL divergence. Extensive experiments across reasoning and coding benchmarks show that our method matches or outperforms complex trained judges (e.g., AutoJudge), offering superior robustness to domain shifts and eliminating the supervision bottleneck entirely. Code is available at \url{https://github.com/sunshy-1/JuDi}.
\end{abstract}

\section{Introduction}
\label{ssy1122:intro}

Large Language Models (LLMs) have demonstrated remarkable capabilities in reasoning and generation \citep{openai-o1,deepseekr1,DBLP:arxiv/AaronAbhimanyu24,DBLP:arxiv/yangli25}, yet their deployment is severely constrained by the high latency and memory bandwidth costs of autoregressive decoding \citep{DBLP:journal/TIST/NaveedKhan23}. To mitigate this, Speculative Decoding (SD) has emerged as a standard acceleration paradigm \cite{DBLP:conf/icml/leviathankalman23, DBLP:arxiv/chenborgeaud23,DBLP:conf/acl/xiayang24,DBLP:journals/corr/abs-2509-04474}. By leveraging a smaller draft model to propose candidate token sequences that are verified in parallel by the larger target model, SD converts memory-bound sequential generation into compute-bound efficient parallel verification.

Standard SD, however, adheres to a strict lossless criterion to guarantee exact target-model sampling. It rejects any draft token that does not strictly align with the target model's distribution. This verification is often overly conservative, rejecting semantically equivalent tokens (e.g., ``6+1=7'' vs. ``6 plus 1 equals 7'') and limiting potential speedups. To address this, recent research has pivoted towards \textit{Judge Decoding} \citep{DBLP:conf/iclr/BachmannAnagnostidis25}, a lossy variant that employs a learned classifier (a ``judge'') to assess the semantic validity of draft tokens. These methods rely on expensive supervision (detailed in Section \ref{ssy1122:revisit_judge_decoding}), including manual annotation~\cite{DBLP:conf/iclr/BachmannAnagnostidis25} or heuristic mining~\cite{DBLP:conf/nips/GaripovVelikonivtsev25,DBLP:journals/corr/abs-2510-02329,DBLP:journals/corr/abs-2505-21600}, to teach a classifier to distinguish between ``critical'' errors and harmless deviations.
\vspace{-6pt}
\begin{figure}[h]
    \centering
    \includegraphics[width=1.01\linewidth]{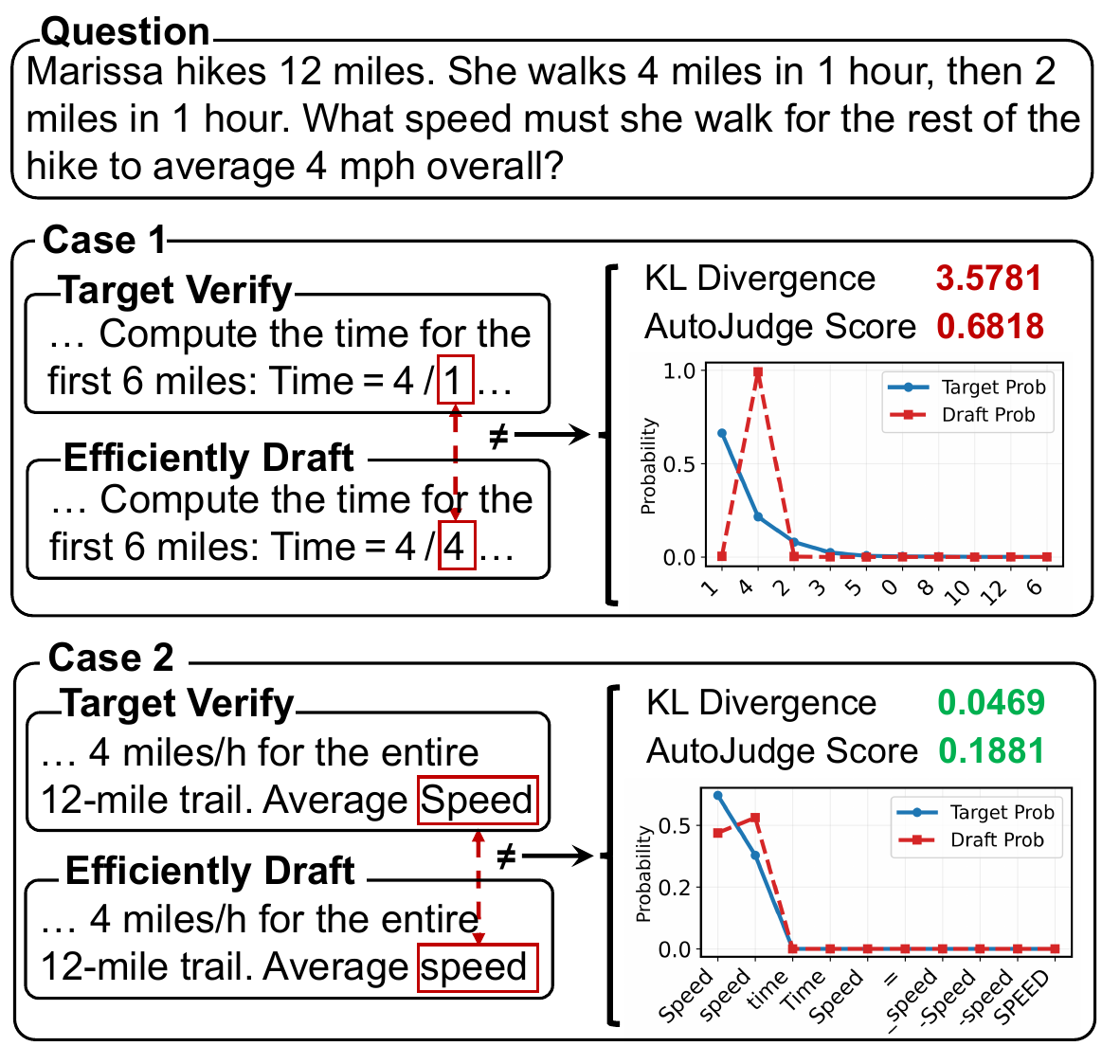}
    \caption{Relationship between AutoJudge score and token-level KL divergence for token criticality in illustrative cases. Higher AutoJudge scores indicate more critical tokens, which coincide with larger KL divergence and stronger target--draft disagreement (Case 1: inconsistent key numbers). Conversely, low scores correspond to small KL divergence and minor, non-semantic deviations (Case 2: capitalization only).}
    \label{ssy1122:relationship_reveal}
\end{figure}

Despite the empirical success of judge decoding, the underlying nature of the ``criticality'' captured by these methods remains opaque. In this work, we investigate the relationship between the learned scoring mechanism of the classifier and the intrinsic distributional statistics of the draft and target models, uncovering a fundamental connection that prior works have overlooked. As illustrated in Figure \ref{ssy1122:relationship_reveal}, taking AutoJudge \citep{DBLP:conf/nips/GaripovVelikonivtsev25} as a representative instance, we observe that its assigned criticality scores exhibit a high degree of alignment with the token-level Kullback-Leibler (KL) divergence. In Case 1, where the draft model introduces a factual error by miscalculating the time (predicting ``$4/4$'' instead of the correct ``$4/1$''), this discrepancy manifests as a sharp distributional shift between the draft and target models, resulting in both a high AutoJudge score and a large KL divergence. Conversely, in Case 2, a non-semantic capitalization difference triggers neither a high judge score nor a significant distributional divergence.\\
\indent These observations motivate a first-principle hypothesis: \textit{the ``criticality'' learned by complex judges is already inherently encoded in the draft–target distributional divergence.} We argue that the scoring function acquired through costly supervision (e.g., in AutoJudge) is essentially a proxy for intrinsic model disagreement. By bridging learned scoring mechanisms with these intrinsic statistics, we propose that a simple, training-free distributional check can effectively substitute for judges that rely on complex supervision signal mining. We ground this proposal theoretically by proving that learned judges and intrinsic divergence share identical logit primitives (linear vs. quadratic), suggesting that the efficacy of expensive supervision can be achieved through purely statistical means while naturally avoiding the generalization fragility of task-specific training. Our contributions are summarized as follows: 
\begin{enumerate}[itemsep=0pt, parsep=0pt, topsep=0pt]
    \item We empirically revisit judge decoding and find that supervision-trained criticality scores are strongly aligned with draft--target distributional disagreement at the token level, suggesting the ``criticality'' signal is model-intrinsic.
    \item We theoretically establish a structural correspondence between learned judges and KL divergence, proving that the learned ``criticality'' relies on the same logit primitives as intrinsic divergence metrics.
    \item Guided by our theoretical findings, we propose a training-free verification method. Empirical validation confirms that simple distributional metrics (e.g., KL divergence) effectively replace costly trained classifiers, achieving comparable or superior performance.
\end{enumerate}

\begin{figure*}
    \centering
    \includegraphics[width=1.0\linewidth]{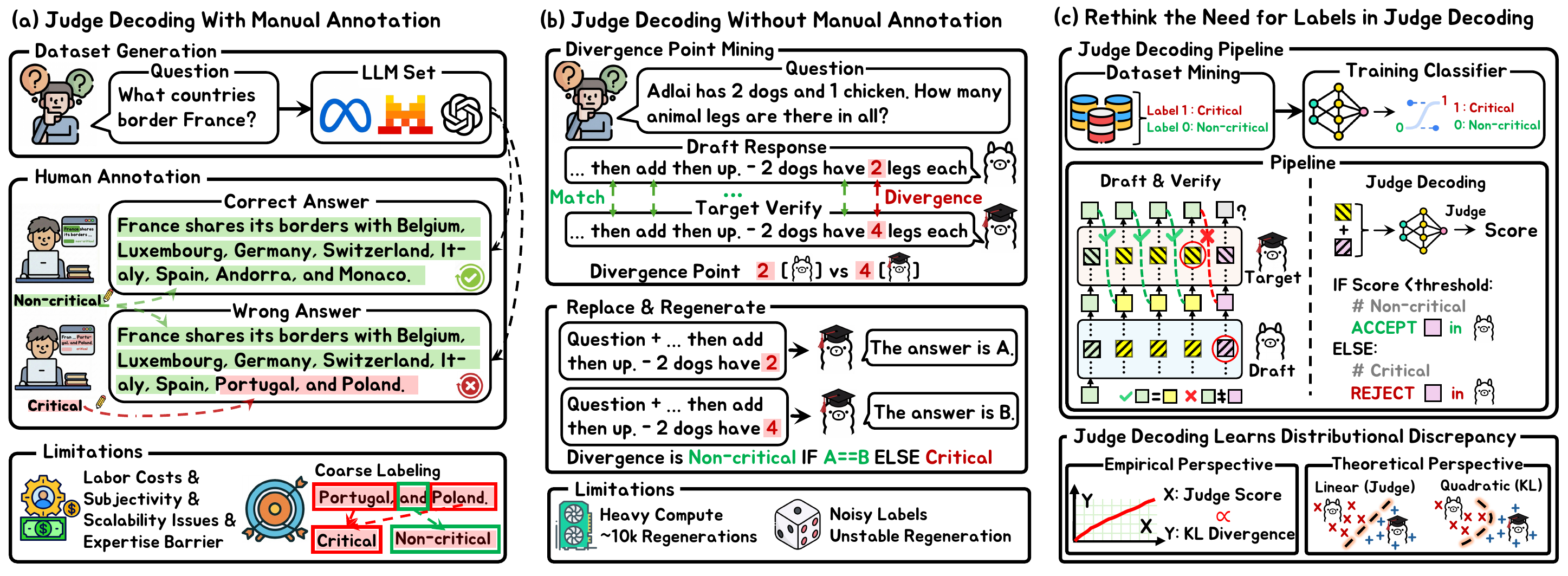}
    \caption{Judge decoding as learning distributional discrepancy between draft and target models under three paradigms. (a) Manual-annotation pipeline for training critical-token judges; (b) Divergence-point mining without manual labels; (c) We show that a linear judge's score is empirically correlated with and theoretically connected to distributional divergence (e.g., KL divergence).}
    \label{ssy1122:framework}
\end{figure*}

\section{Background}
\label{ssy1122:background}
To contextualize our approach, we briefly review the paradigm of speculative decoding and its recent evolution towards semantic-level verification, specifically judge decoding.

\subsection{Speculative Decoding}
\label{ssy1122:speculative_decoding}
To accelerate the memory-bound autoregressive inference of LLMs, SD employs a \textit{draft-then-verify} strategy to increase concurrency~\cite{DBLP:arxiv/chenborgeaud23, DBLP:conf/icml/leviathankalman23}.
Instead of generating one token at a time, a smaller \emph{draft} model $\mathcal{M}_D$ first autoregressively produces a candidate block of $\gamma$ tokens, $\mathbf{d}_{t:t+\gamma-1}$, given a prefix $\mathbf{x}_{<t}$. 
The large \emph{target} model $\mathcal{M}_T$ then evaluates all candidates in a single parallel forward pass. 
To ensure the output distribution matches $\mathcal{M}_T$ exactly (lossless decoding), a rejection sampling mechanism is applied: a token $d_{t+i}$ is accepted with probability $\min(1, \frac{P_T(d_{t+i})}{P_D(d_{t+i})})$. 
This paradigm allows the system to generate multiple tokens per expensive model call, effectively converting memory-bound sequential operations into compute-bound parallel verification.

\subsection{Judge Decoding} 
\label{ssy1122:judge_decoding}
While SD guarantees mathematical equivalence to the target distribution, this strict alignment can be overly conservative. For instance, a draft ``6+1=7'' might be rejected if the target model prefers ``6 plus 1 equals 7'', despite their semantic equivalence. This strictness limits the effective speedup.

To address this, recent works have proposed judge decoding~\cite{DBLP:conf/iclr/BachmannAnagnostidis25,DBLP:conf/nips/GaripovVelikonivtsev25}. 
The core motivation stems from the insight that a model's reaction to processing a token reveals more than just softmax probabilities. Specifically, the last hidden layer embeddings of erroneous tokens effectively ``flag'' errors and contradictions, reflecting the model's intrinsic tendency to rectify mistakes immediately~\cite{DBLP:conf/iclr/BachmannAnagnostidis25,DBLP:conf/acl/ServedioBPAN25,DBLP:conf/emnlp/AzariaM23}.
Leveraging this latent error-signaling behavior, judge decoding employs a learnable verifier (or ``judge'') $f_\phi$, typically a lightweight classifier on top of the embeddings $\mathbf{h}$, to assess the validity of the drafted block. 
A candidate is accepted if $f_\phi(\mathbf{h}) > \tau$ for a predefined threshold $\tau$. By shifting from rigid probability comparisons to analyzing these internal validity signals, judge decoding optimizes the trade-off between throughput and generation fidelity, accepting plausibly correct tokens that standard SD would reject. The pipeline of judge decoding is depicted in Figure \ref{ssy1122:framework} (c).
\section{The Bottlenecks of Supervision in Judge Decoding} \label{ssy1122:revisit_judge_decoding}

While judge decoding represents a paradigm shift in speculative decoding by introducing semantic verification, its efficacy is fundamentally constrained by the quality and acquisition cost of the supervision signals used to train the judge. As shown in Figure \ref{ssy1122:framework}, current approaches typically derive supervision either from expensive manual annotation or heuristic model rollouts, inevitably inheriting the generalization fragility of task-specific training. This section critically examines these two paradigms to highlight the trade-offs between annotation cost, label granularity, and signal stability.

\subsection{Judge Decoding With Manual Annotation}

\citet{DBLP:conf/iclr/BachmannAnagnostidis25} introduced the judge decoding framework as a verification paradigm that augments the target model with a lightweight classifier to assess token acceptability. This design effectively highlights the promise of learning correctness-aware verification beyond traditional probability alignment. However, two critical limitations hinder its widespread adoption:
\begin{enumerate}[itemsep=0pt, parsep=2pt, topsep=0pt]
    \item \textbf{Costly and Misaligned Supervision.} Training relies on manual annotations of factual or reasoning mistakes. This process is not only prohibitively expensive but also prone to \emph{alignment mismatch}: human annotators judge plausibility based on external knowledge, which may differ from the model's intrinsic ``knowledge boundary''. A token might be factually true but out-of-distribution for the model, or vice versa.    
    

    \item \textbf{Coarse-grained Span Labeling.} Existing schemes typically employ a span-based labeling strategy: once an error occurs, all subsequent tokens are labeled as \textcolor{myRed}{{Critical Tokens}} (negative samples). As shown in Figure \ref{ssy1122:labeled_examples}, this design fails to isolate the \textcolor{myRed}{\underline{Logic-pivoting Token}}—the specific token that triggers the reasoning shift. Once the answer includes ``Portugal'', every subsequent token is labeled as critical, even though only “Portugal” and ``Poland'' are the true logic‑pivoting tokens.

\end{enumerate}

\begin{figure}[htbp]
    \centering
    \includegraphics[width=1.0\linewidth]{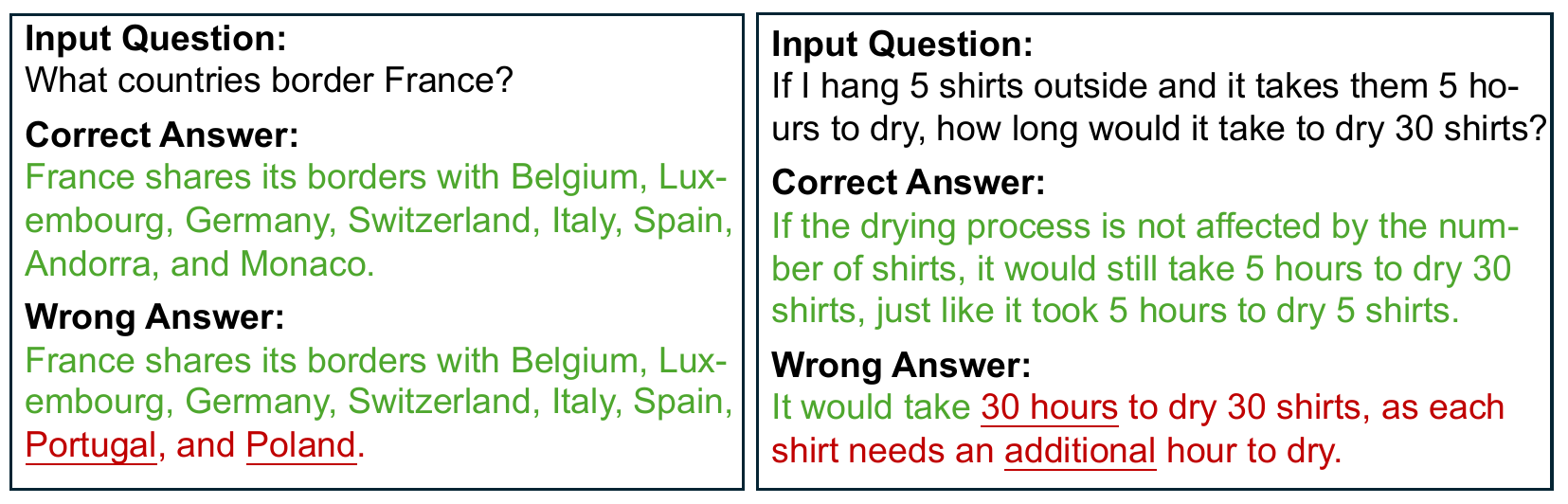}
    \caption{Illustration of overextended labeling boundaries in manual annotation, adapted from \citet{DBLP:conf/iclr/BachmannAnagnostidis25}. Key: \textcolor{myGreen}{{Non-critical Tokens}}, \textcolor{myRed}{{Critical Tokens}}, {\textcolor{myRed}{\underline{Logic-pivoting Tokens}}}. The span-based labeling (\textcolor{myRed}{red}) obscures the true logic-pivoting tokens (\textcolor{myRed}{\underline{red underline}}), creating noisy supervision signals.}
    \label{ssy1122:labeled_examples}
\end{figure}

\subsection{Judge Decoding Without Manual Annotation}
To mitigate human dependency, \citet{DBLP:conf/nips/GaripovVelikonivtsev25} proposed AutoJudge, which mines critical tokens by checking if a token substitution alters the final answer. While this approach automates data collection, it introduces new computational and statistical challenges:
\begin{enumerate}[itemsep=0pt, parsep=2pt, topsep=0pt]
    \item \textbf{Prohibitive Mining Costs.} AutoJudge necessitates counterfactual rollouts at each divergence point. Figure \ref{ssy1122:analysis_running_time} shows that mining supervision for the GSM8K dataset with a 70B model consumes more than 2,700 GPU hours. This prohibitive cost of data mining undermines the efficiency benefits of speculative decoding and hinders scalability, implicitly constraining the judge's generalization capabilities due to the difficulty of covering diverse data distributions.
    \item \textbf{Stochastic Instability.} LLM generation is inherently stochastic. A token labeled as ``critical'' in one rollout might be deemed ``non-critical'' in another simply due to inherent non-determinism in the generation process. Figure \ref{ssy1122:consistency_bon} reveals that only 25.4\% of tokens maintain consistent criticality across four trials. This high variance injects significant label noise, causing the judge to learn spurious correlations rather than distinguishing between \emph{causal logic errors} and \emph{aleatoric uncertainty}.
\end{enumerate}

\begin{figure}[htbp]
    \centering
    \includegraphics[width=1.02\linewidth]{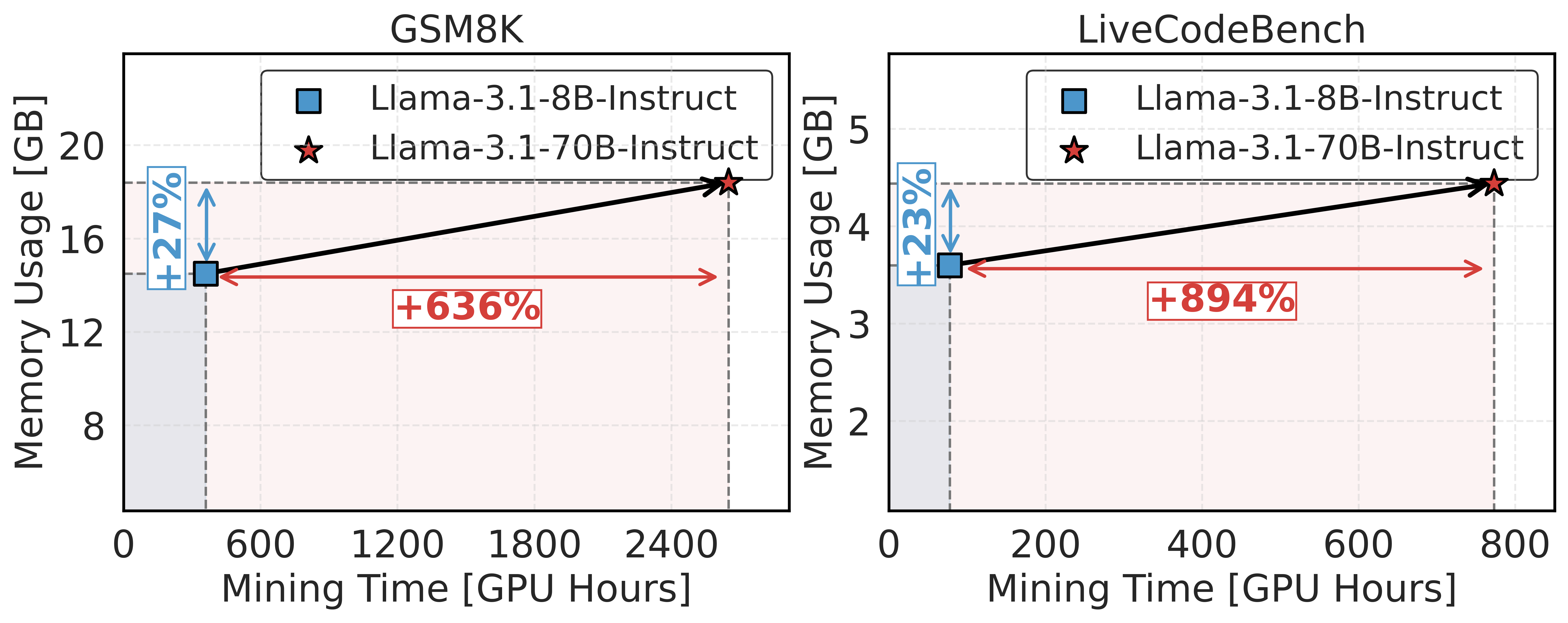}
    \caption{Efficiency analysis of AutoJudge dataset mining. The substantial growth in GPU hours and memory footprint for larger models highlights a significant scalability bottleneck.}
    \label{ssy1122:analysis_running_time}
\end{figure}

\begin{figure}[htbp]
    \centering
    \includegraphics[width=1.02\linewidth]{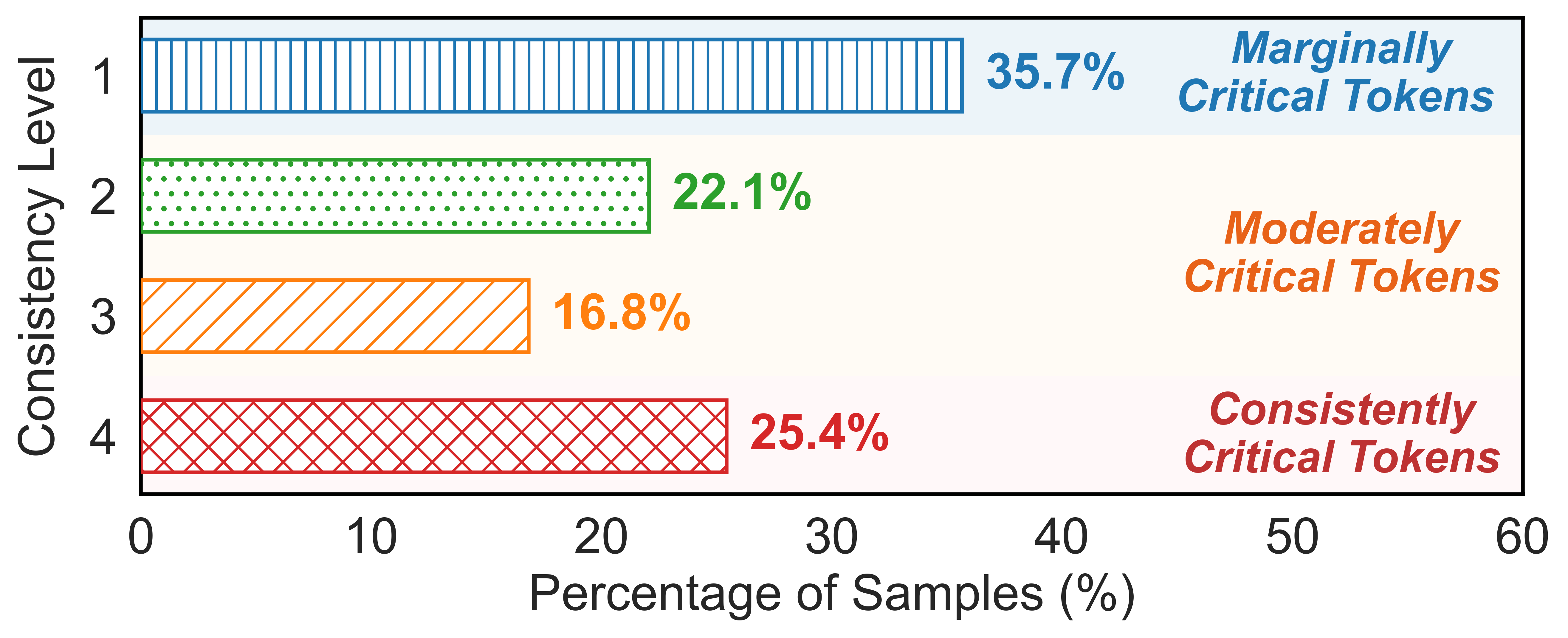}
    \caption{Consistency analysis of critical tokens. The low percentage of consistently critical tokens (Level 4) indicates that heuristic mining is heavily influenced by generation randomness.}
    \label{ssy1122:consistency_bon}
\end{figure}

\noindent\textbf{Summary.} Current paradigms face a dilemma: manual annotation provides stable labels but is labor-intensive and coarse, while heuristic mining offers automation but is plagued by high computational costs and stochastic noise. These limitations underscore the need for an efficient and robust alternative, motivating the following section.
\section{Can We Bypass Data Annotation?}
\label{ssy1122:rethinking_supervision}
The revisiting analysis in Section \ref{ssy1122:revisit_judge_decoding} exposes critical bottlenecks in current paradigms, thereby naturally leading to a key question: \textit{Must the identification of critical tokens rely on external supervision, or can such signals be revealed directly through the models themselves?}

\subsection{Empirical Perspective}
To answer the above question, we compare the ``criticality'' captured by the classifier in AutoJudge against the intrinsic probabilistic states derived from the target distribution $\mathbf{p}_t = \mathcal{M}_T(\mathbf{x}_{<t})$ and the draft distribution $\mathbf{q}_t = \mathcal{M}_D(\mathbf{x}_{<t})$. We quantify these states using the following two metrics: 

\begin{enumerate}[itemsep=0pt, parsep=2pt, topsep=0pt]
\item{ 
    \textbf{Uncertainty (Entropy)} quantifies the predictive uncertainty of the model. For any distribution $\pi \in \{\mathbf{p}_t, \mathbf{q}_t\}$, the calculation of entropy is defined as:
    \begin{equation}
        H(\pi) = - \sum_{v \in \mathcal{V}} \pi(v) \log \pi(v),
    \end{equation}
    where higher entropy indicates a more dispersed probability distribution over the vocabulary $\mathcal{V}$.}

\item{ 
    \textbf{Disagreement (KL Divergence)} captures the structural deviation between the draft and target reasoning paths via:
    \begin{equation}
        D_{\text{KL}}(\mathbf{p}_t || \mathbf{q}_t) = \sum_{v \in \mathcal{V}} \mathbf{p}_t(v) \log \frac{\mathbf{p}_t(v)}{\mathbf{q}_t(v)},
        \label{eq:kl_divergence}
    \end{equation}
    which quantifies the distributional disagreement between the two models.}
\end{enumerate}

\smallskip
Using these metrics, we investigate the correlation between the learned criticality scores from AutoJudge and intrinsic statistics on the GSM8K dataset. Specifically, we stratify the classifier's output probabilities, where higher values indicate greater importance to the final answer, into equally spaced bins and compute the average value of the intrinsic indicators within each bin. As shown in Figure \ref{ssy1122:analysis_kl_trend}, we observe a strong positive correlation: the average KL divergence rises consistently as the classifier assigns higher criticality scores. This implies that influential tokens typically coincide with regions of high distributional disagreement between the target and draft models. Consequently, the token-level influence that AutoJudge learns to approximate via costly supervision appears to be inherently encoded in the divergence between the models' probability distributions.

\begin{figure}[htbp]
    \centering
    \includegraphics[width=1.0\linewidth]{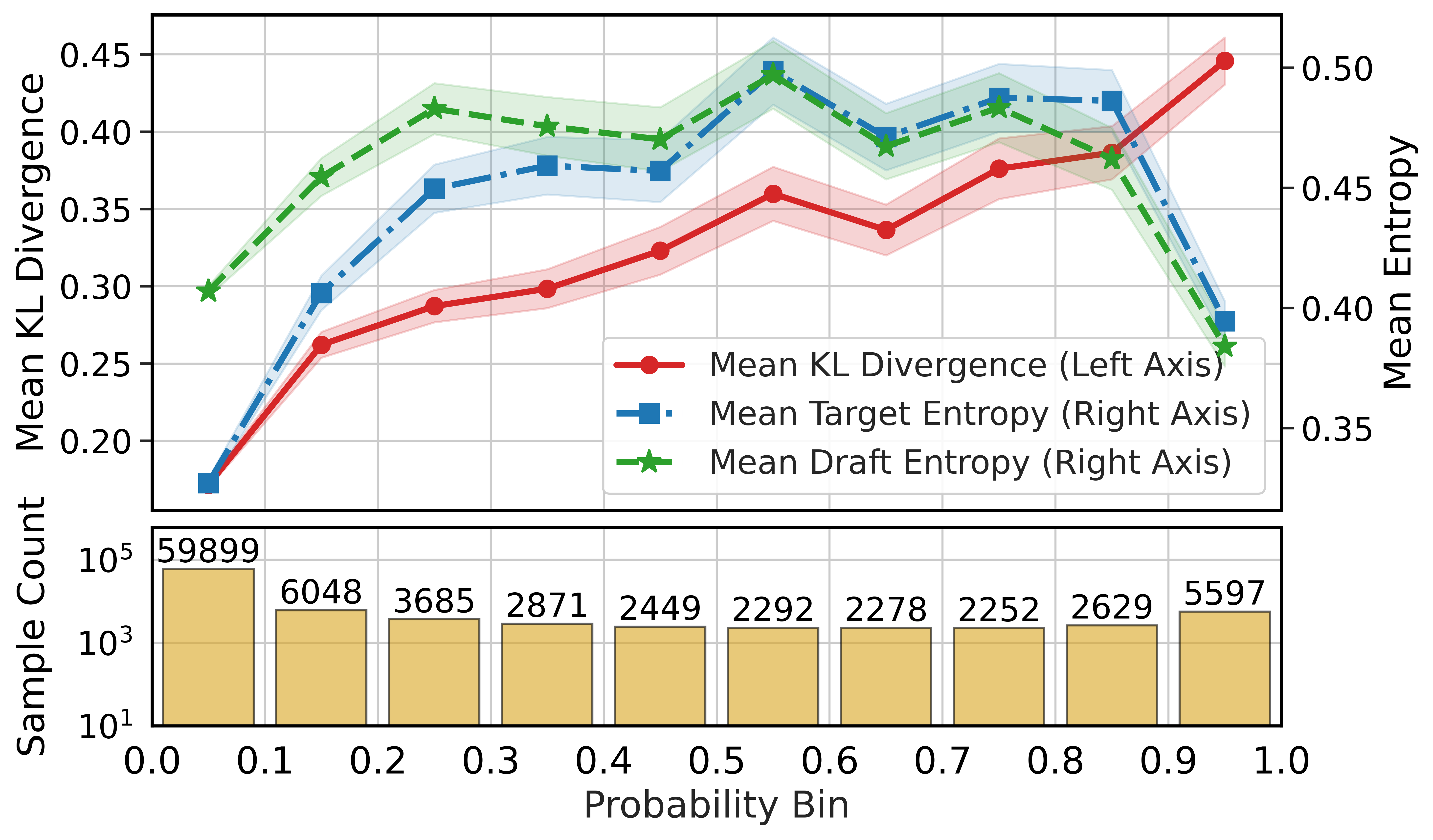}
    \caption{Correlation between supervised criticality scores and intrinsic model statistics on GSM8K. Samples are stratified into 10 bins based on AutoJudge scores (x-axis). \textbf{Top:} The mean KL divergence (red line) shows a clear upward trend as criticality increases, whereas model entropy (dashed lines) shows no significant correlation. \textbf{Bottom:} The sample count distribution across probability bins (i.e., AutoJudge scores).}
    \label{ssy1122:analysis_kl_trend}
\end{figure}

\fakeparagraph{Key Insight}
The observed correlation suggests an intrinsic connection between criticality and predictive uncertainty, indicating that the models' existing probabilistic behavior already conveys informative cues for identifying critical tokens without additional supervision. Thus, KL divergence serves as a training-free proxy for locating logic pivots. To further enhance robustness, we employ a lightweight confidence mask: when the target model's top-1 probability exceeds 0.9, the system defaults to standard speculative sampling, reserving KL-based relaxation for tokens with genuine uncertainty. We retain this fallback as a principled, zero-cost safeguard; a detailed analysis is deferred to Appendix~\ref{appendix:top1_fallback}.

\subsection{Theoretical Perspective}\label{sec:theory_analysis}
This section characterizes the theoretical alignment between KL-based thresholding and the AutoJudge linear classifier. We first summarize the main intuition in the takeaway below, and then formalize it in Theorem~\ref{ssy1122:structural_correspondence}.

\begin{takeaway}{
AutoJudge and KL-based thresholding rely on the same signal, namely how much the target model shifts the draft model’s relative token preferences, expressed via $\Delta_{ij}(x)$ (defined in Eq.~\eqref{ssy1122:formulation_shift}).
AutoJudge learns a \emph{linear} decision rule over these shifts, while the KL divergence provides a \emph{quadratic} aggregation of their overall magnitude. As a result, KL-based thresholding can act as a proxy for the same ``logic pivots'' without dataset mining.}
\end{takeaway}

To facilitate the following analysis, we first introduce some necessary notation. Let $x=[h_t;h_d]$ be the concatenated feature vector, where $h_t, h_d \in \mathbb{R}^d$ denote the hidden states of the target and draft models. The corresponding logits are given by $z_t = W_t h_t + b_t$ and $z_d = W_d h_d + b_d$, yielding the output distributions $P_t=\mathrm{softmax}(z_t)$ and $P_d=\mathrm{softmax}(z_d)$. The theoretical link between the KL measure and the trained classifier is the pairwise logit-gap difference $\Delta_{ij}(x)$, defined for any vocabulary indices $(i,j)$ as
\begin{equation}
\label{ssy1122:formulation_shift}
    \Delta_{ij}(x):=(z_t(i)-z_t(j))-(z_d(i)-z_d(j)),
\end{equation}
a fundamental linear primitive that captures the shift in relative preference between tokens when transitioning from the draft to the target model.
\begin{figure*}[h]
    \centering
    \includegraphics[width=0.49\linewidth]{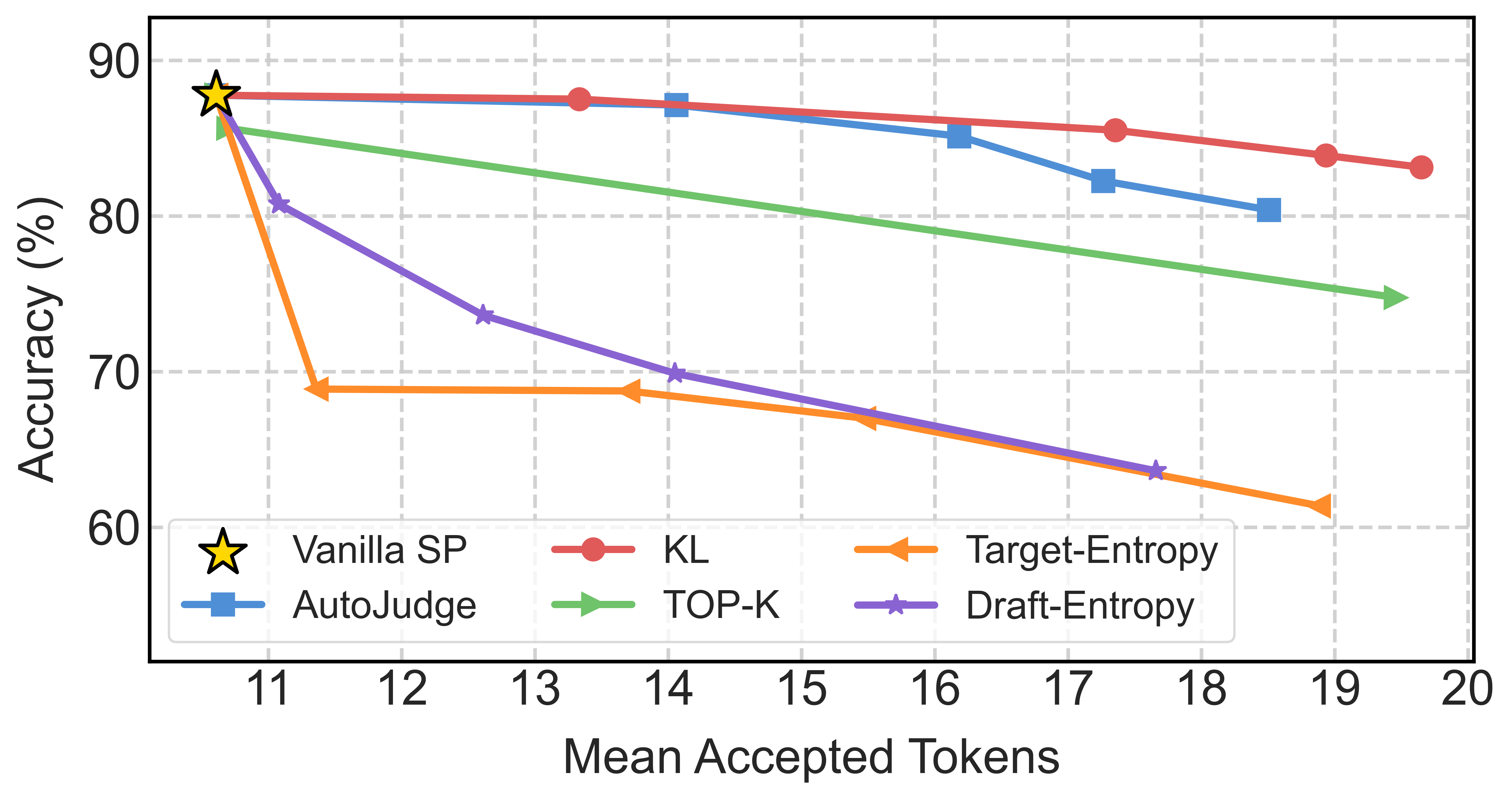}\hspace{3pt}
    \includegraphics[width=0.49\linewidth]{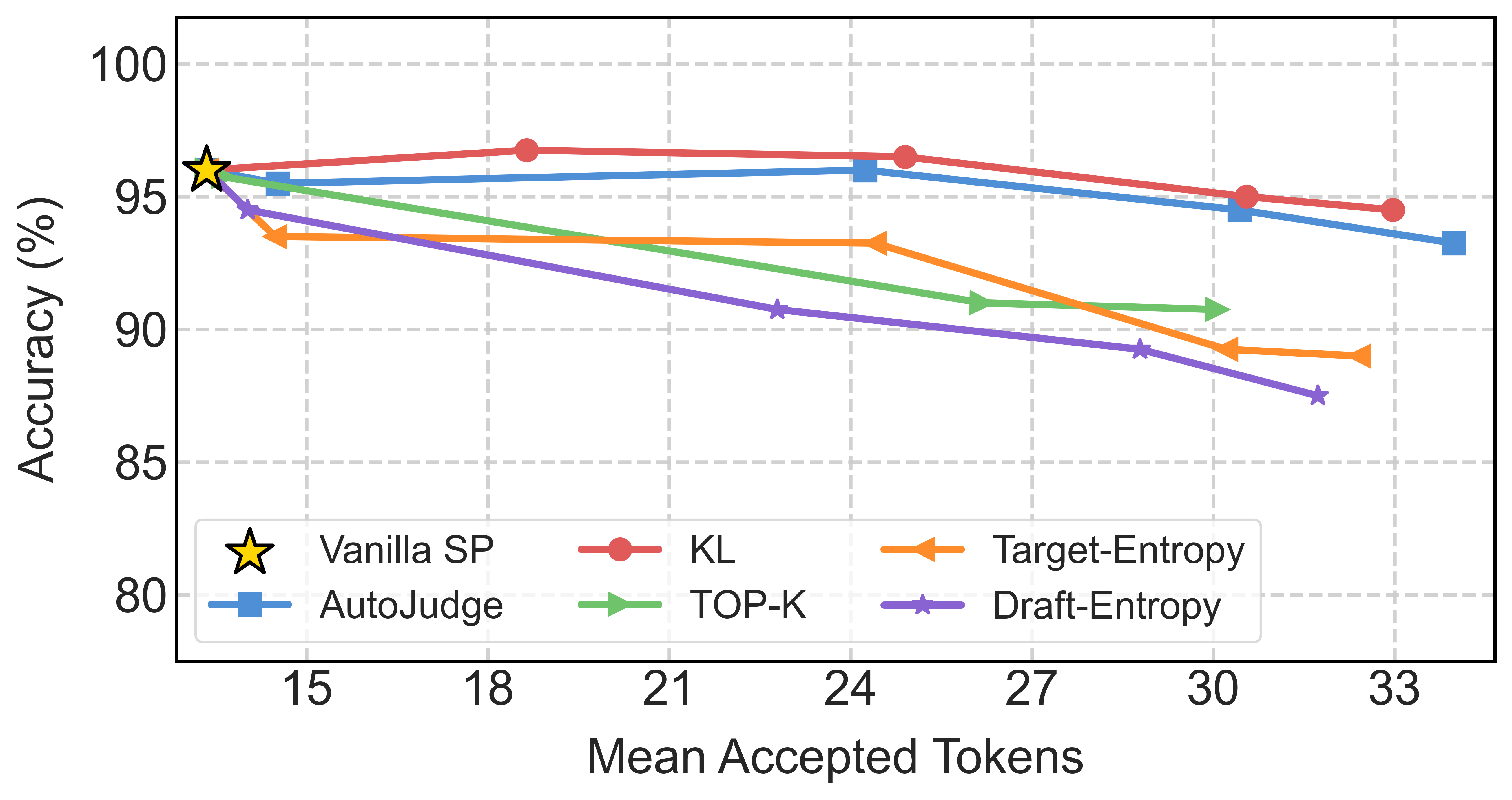}
    \caption{Accuracy and MAT on GSM8K for Llama-3.2-1B-Instruct/Llama-3.1-8B-Instruct  (left) and Llama-3.1-8B-Instruct/Llama-3.1-70B-Instruct (right).}
    \label{ssy1122:performace_gsm8k}
\end{figure*}
\begin{figure*}[h]
    \centering
    \includegraphics[width=0.49\linewidth]{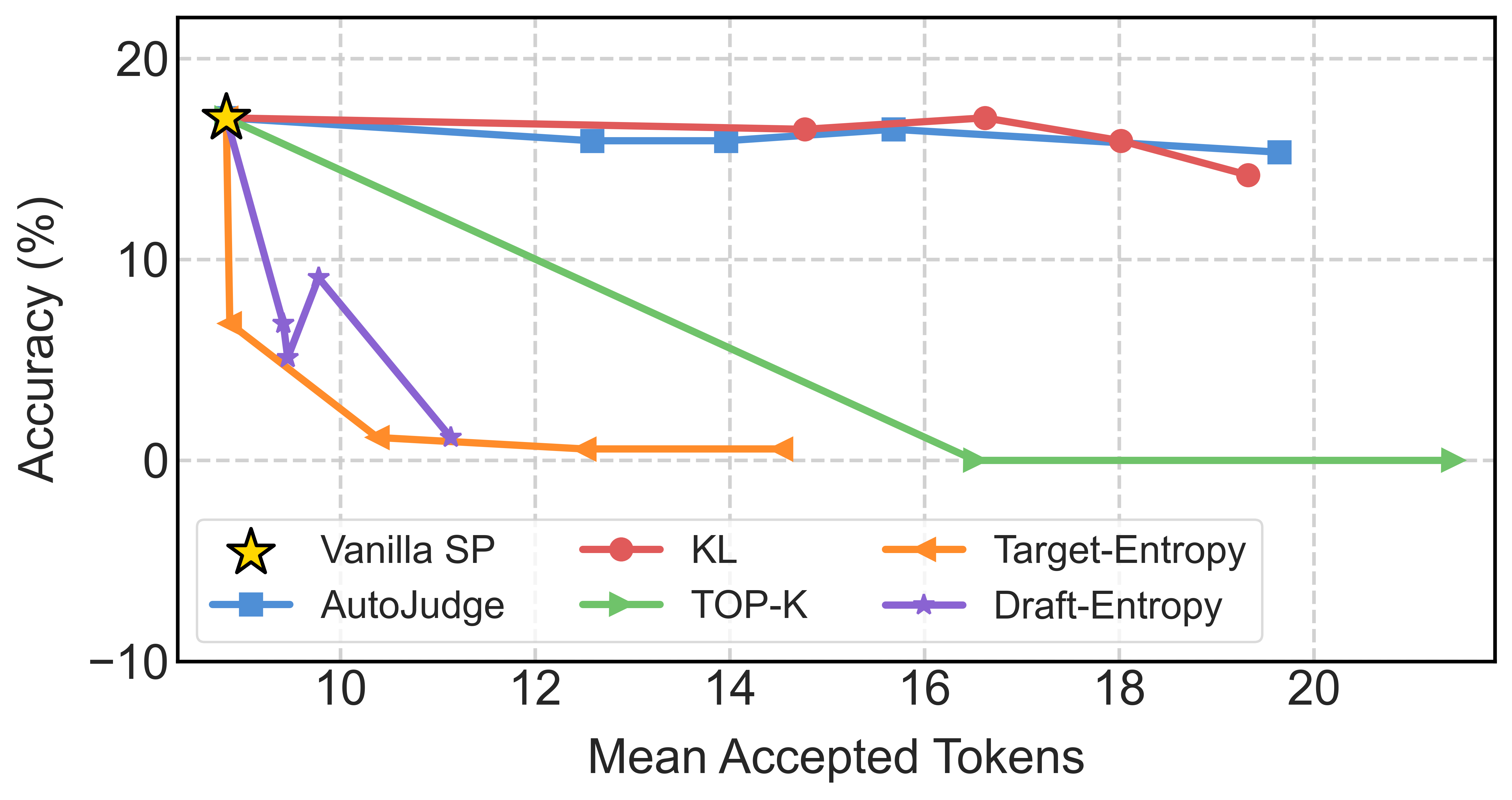}\hspace{3pt}
    \includegraphics[width=0.49\linewidth]{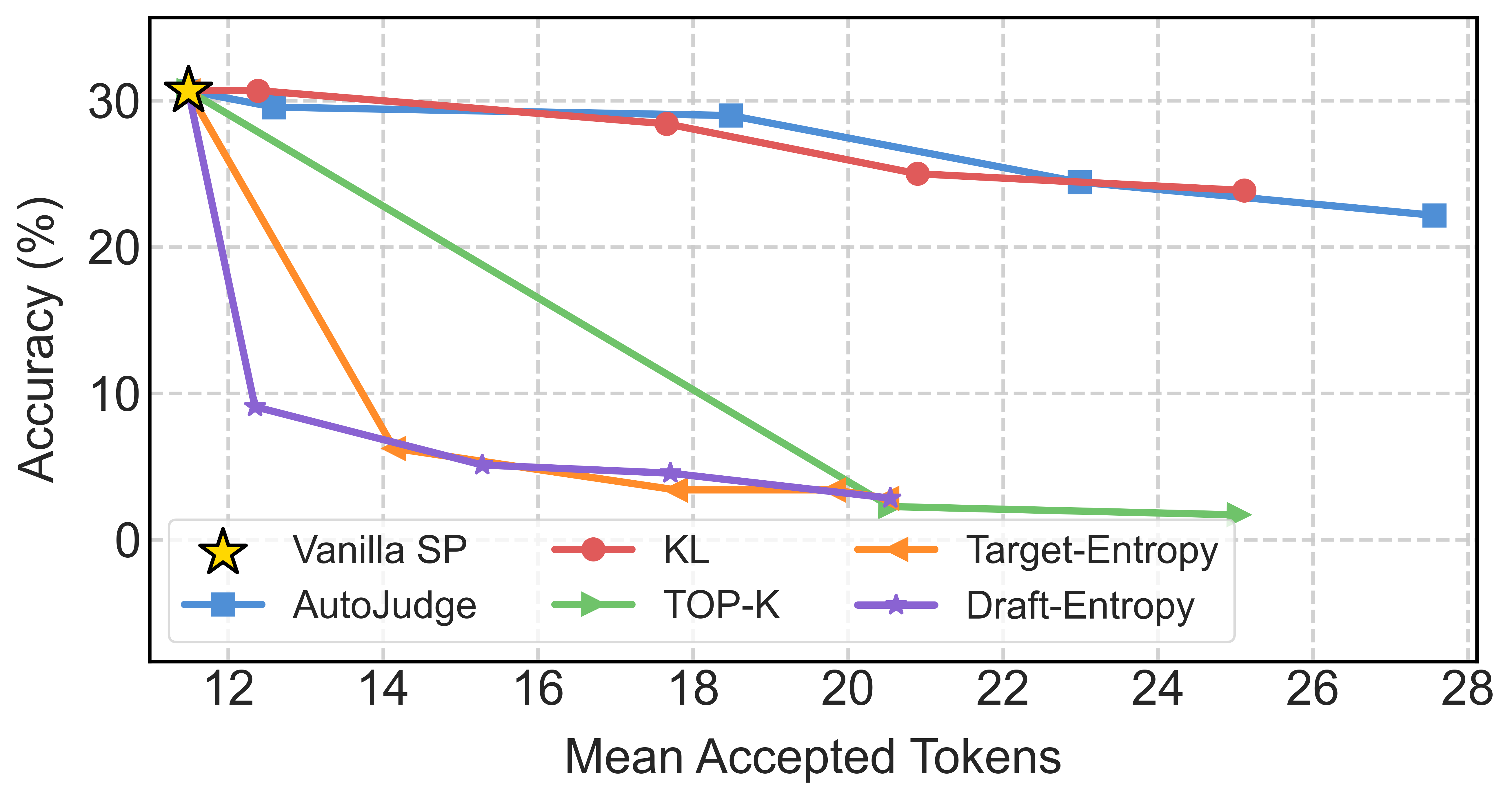}
    \caption{Accuracy and MAT on LiveCodeBench for Llama-3.2-1B-Instruct/Llama-3.1-8B-Instruct (left)  and Llama-3.1-8B-Instruct/Llama-3.1-70B-Instruct (right).}
    \label{ssy1122:performace_lcb}
    \vspace{-9pt}
\end{figure*}
\begin{theorem}[Structural Correspondence of KL and Linear Classifiers]
\label{ssy1122:structural_correspondence}
The empirical alignment between KL-based thresholding and the trained linear classification stems from their shared dependence on the primitives $\Delta_{ij}(x)$: (i) \textbf{KL Divergence as Quadratic Aggregation:} Under a second-order expansion, the KL divergence acts as a weighted quadratic sum of the primitives $\Delta_{ij}(x)$. In this view, KL-based thresholding defines a quadratic decision boundary in the primitive space, sensing the total magnitude of distribution shifts. (ii) \textbf{Trained Classifier as Linear Partitioning:} The linear classifier identifies critical tokens by defining a linear decision surface in the space of primitives $\Delta_{ij}(x)$, partitioning the feature space based on whether token mismatches lead to a deviation in output quality.
\end{theorem}

\begin{proof}
See Appendix \ref{sec:a0_setup} for details.
\end{proof}

Theorem \ref{ssy1122:structural_correspondence} attributes the stable agreement between both methods to their shared foundation in the linear primitives $\Delta_{ij}(x)$, differing only in whether they employ quadratic or linear geometry to distinguish token criticality.

\section{Experiment}\label{sec:experiments}
\begin{figure*}
    \centering
    \includegraphics[width=0.49\linewidth]{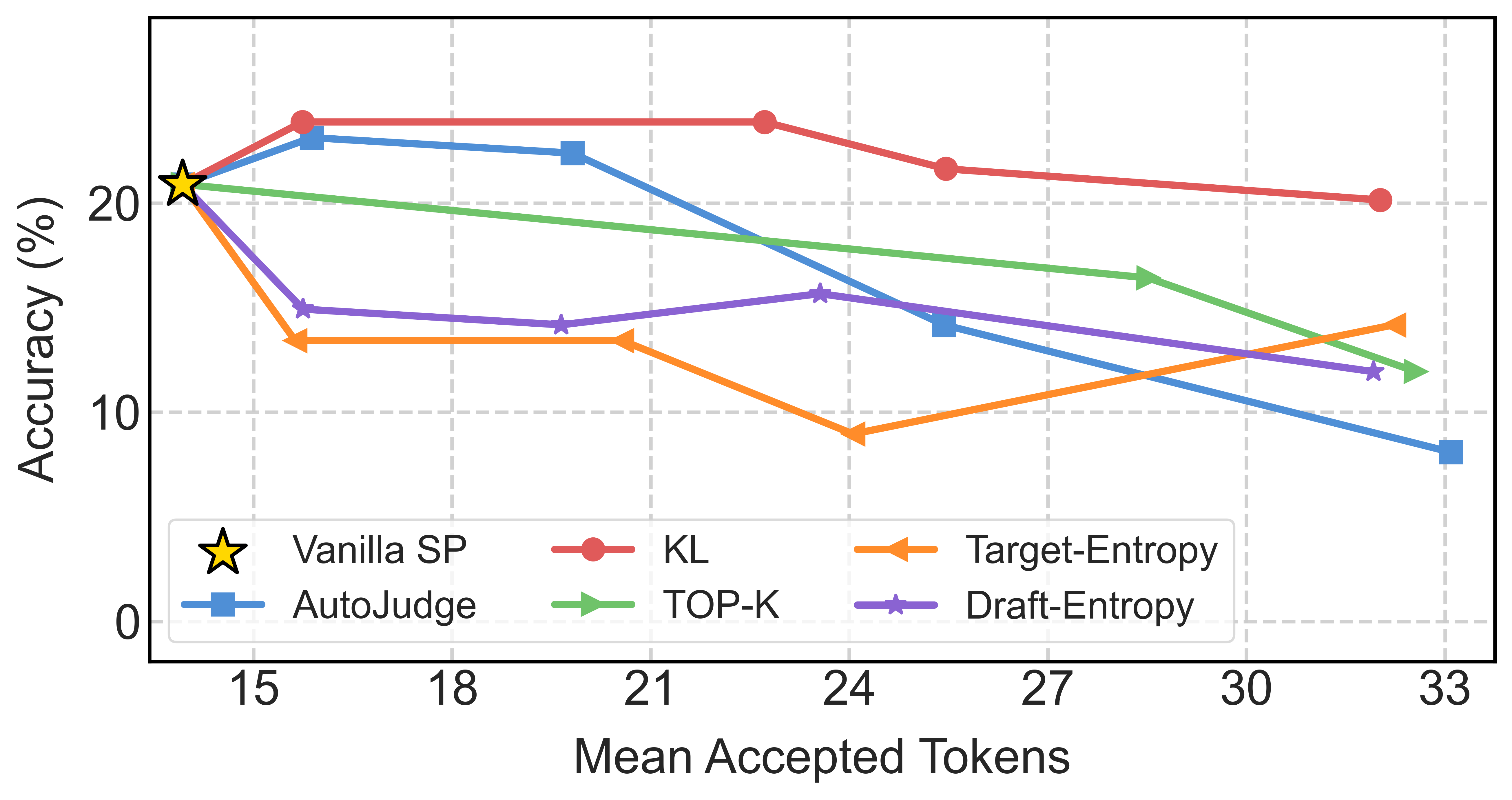}\hspace{3pt}
    \includegraphics[width=0.49\linewidth]{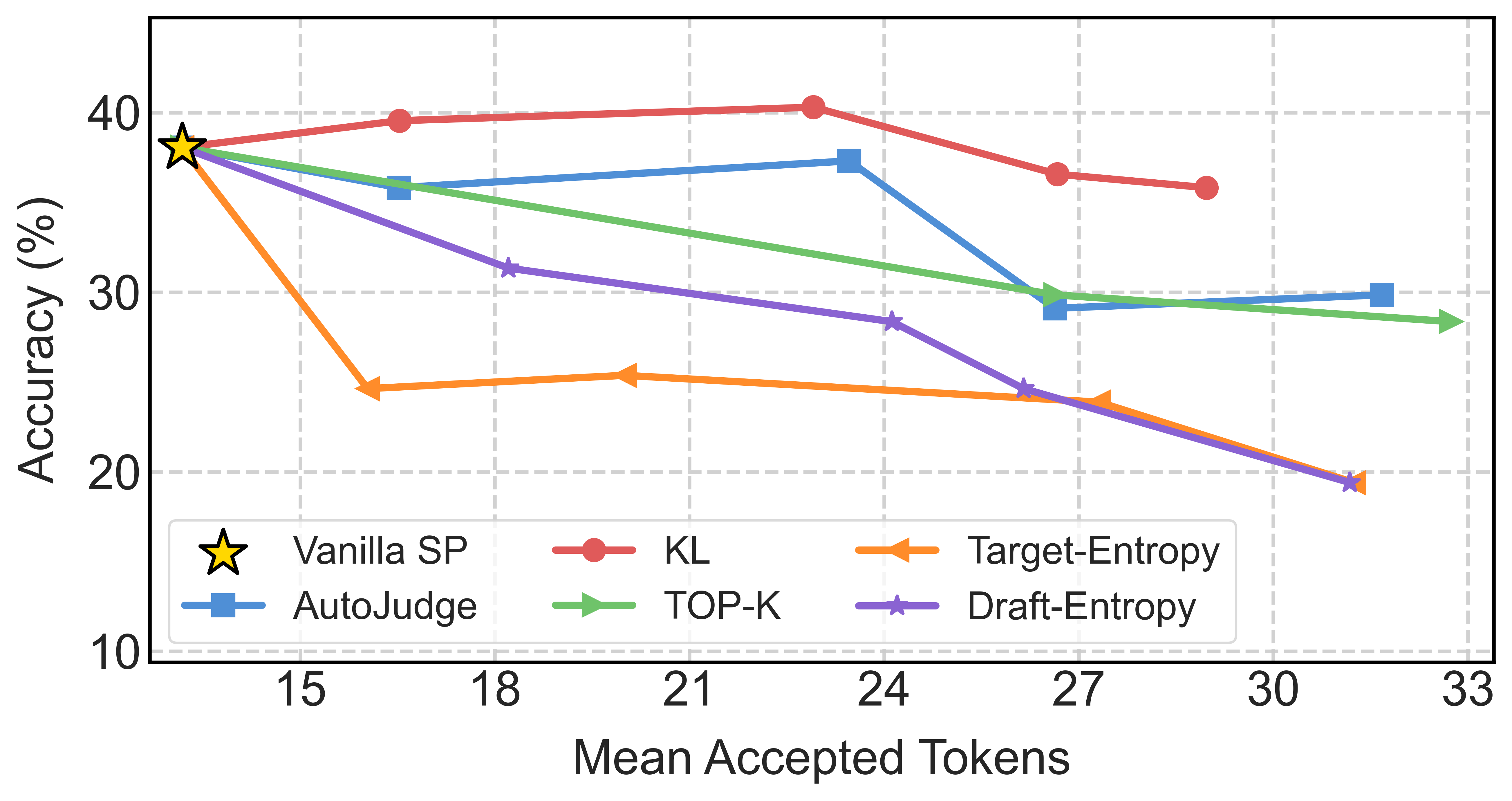}
    \caption{Accuracy and MAT on MATH-500-Hard for Llama-3.2-1B-Instruct/Llama-3.1-8B-Instruct (left) and Llama-3.1-8B-Instruct/Llama-3.1-70B-Instruct (right).}
    \label{ssy1122:performace_math500}
\end{figure*}
\begin{figure*}
    \centering
    \includegraphics[width=0.49\linewidth]{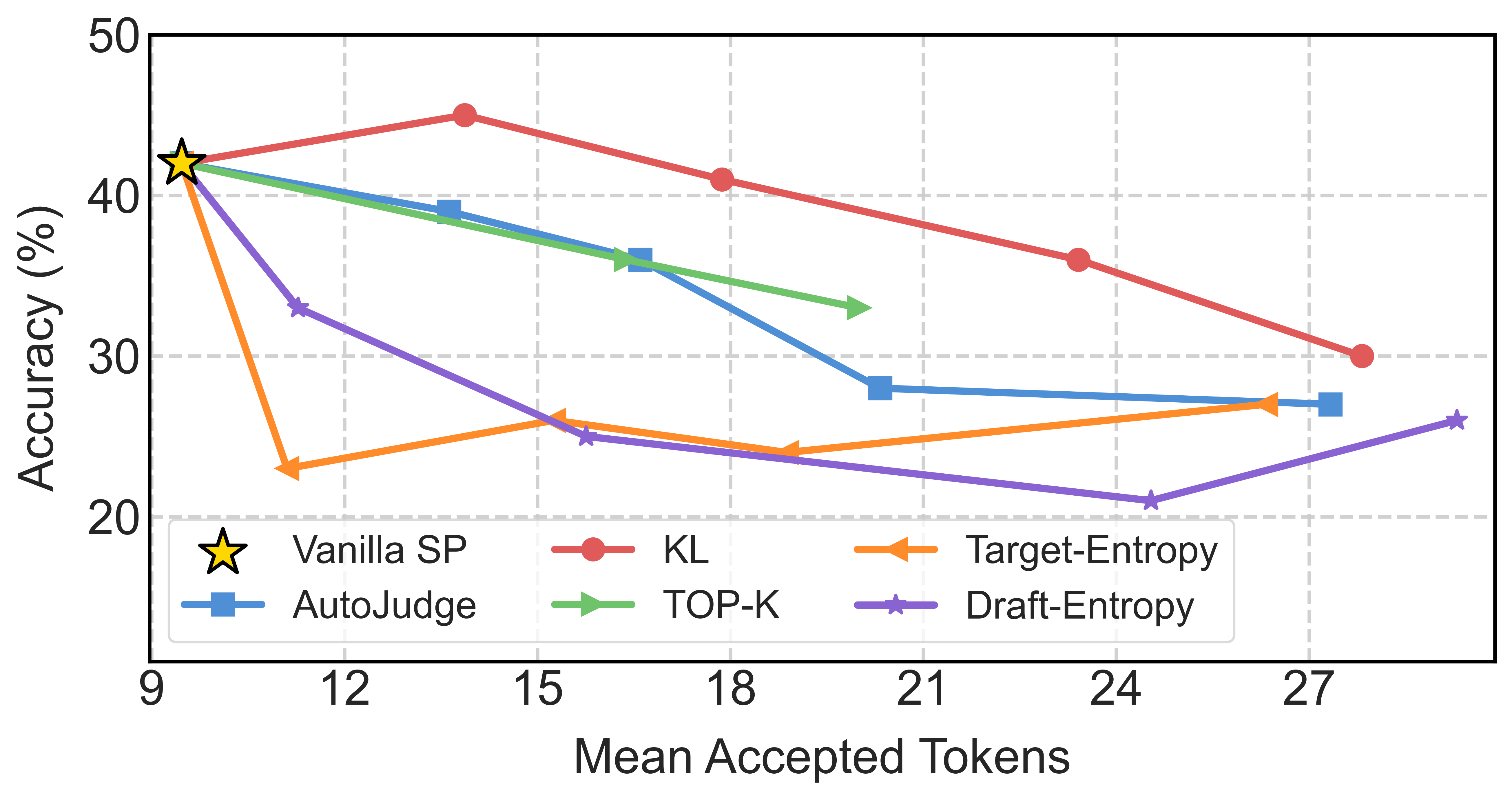}\hspace{3pt}
    \includegraphics[width=0.49\linewidth]{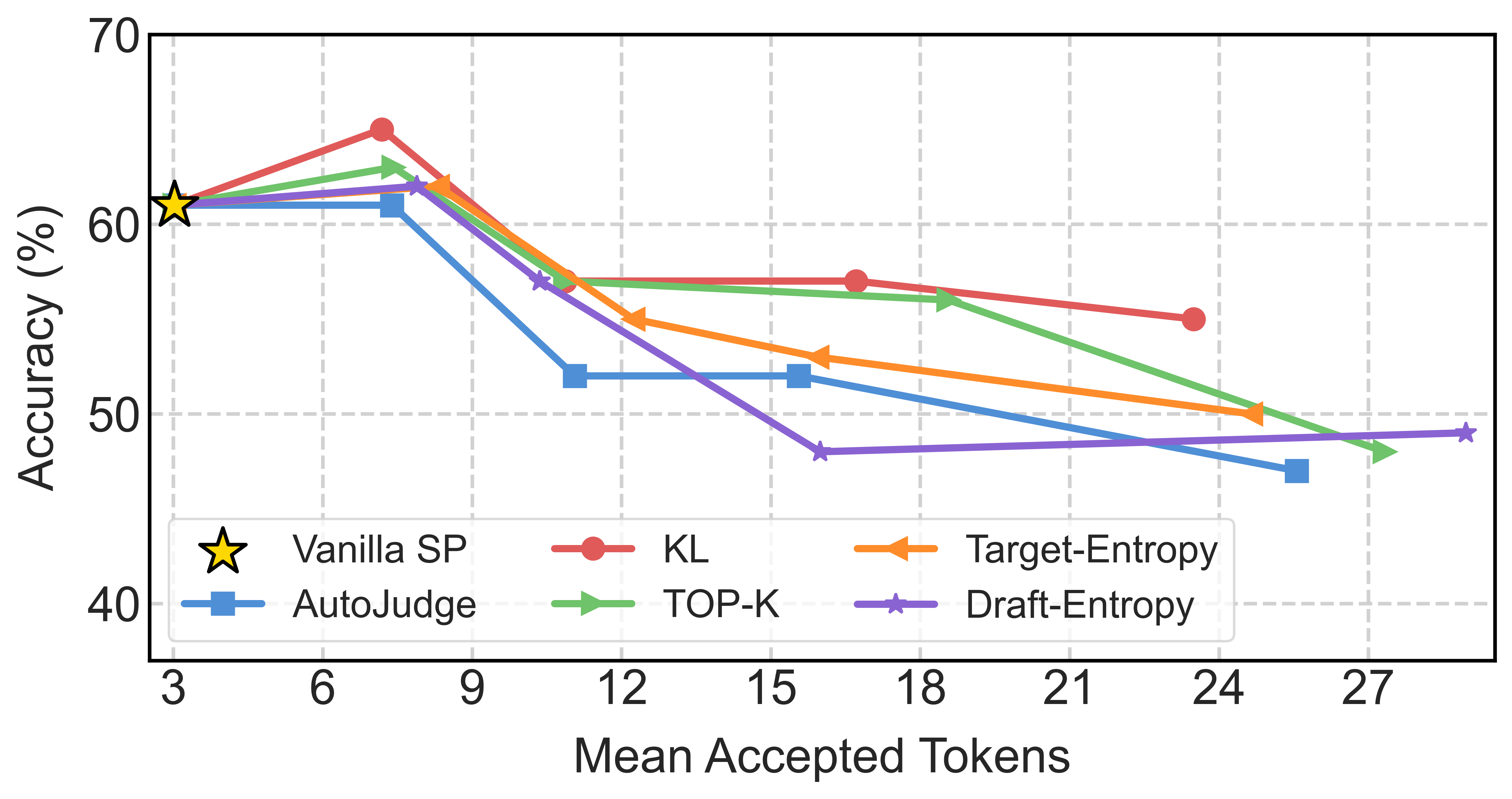}
    \caption{Accuracy and MAT on MMLU-Pro for Llama-3.2-1B-Instruct/Llama-3.1-8B-Instruct (left) and Llama-3.1-8B-Instruct/Llama-3.1-70B-Instruct (right).}
    \label{ssy1122:performace_mmlupro}
\end{figure*}
This section first introduces the experimental setup, benchmarks, and baselines, then presents results showing that training-free KL-based verification is on par with, and sometimes outperforms, the learned judge baseline.

\subsection{Experimental Setup and Baselines}

\paragraph{Benchmarks and Models.}
We adopt standard speculative decoding protocols \citep{DBLP:conf/nips/GaripovVelikonivtsev25} across four benchmarks: GSM8K \citep{DBLP:arxiv/CobbeKosaraju21} and MATH-500-Hard \citep{DBLP:conf/nips/HendrycksBurns21} for mathematics, LiveCodeBench \citep{DBLP:conf/iclr/JainHan25} for coding, and MMLU-Pro \citep{DBLP:conf/nips/WangMa24} for comprehensive knowledge.
Our primary model pairs are Llama‑3.2‑1B-Instruct/Llama‑3.1‑8B-Instruct and Llama‑3.1‑8B-Instruct/Llama‑3.1‑70B-Instruct \citep{DBLP:arxiv/AaronAbhimanyu24}. Additionally, we evaluate the Qwen3‑0.6B/Qwen3‑8B pair \citep{DBLP:arxiv/yangli25} to investigate performance on reasoning-specialized models.

\vspace{3.8pt}\noindent\textbf{Baselines and Metrics.}
We compare our method against: (i) Vanilla SP \citep{DBLP:arxiv/chenborgeaud23} (standard speculative decoding); (ii) Top-K \cite{DBLP:conf/iclr/BachmannAnagnostidis25}, which relaxes verification based on token ranking; (iii) AutoJudge \citep{DBLP:conf/nips/GaripovVelikonivtsev25}, a training-based classifier approach; and two entropy-based methods, (iv) Target-Entropy and (v) Draft-Entropy \cite{DBLP:conf/emnlp/WangTian25}.
Evaluation metrics include Mean Accepted Tokens (MAT) for efficiency, benchmark-specific scores for quality, and end-to-end wall-clock speedup relative to Vanilla SP implemented in vLLM \citep{DBLP:conf/sosp/KwonLi23}. Further details are in Appendix \ref{ssy1122:additional_exp_details}.

\begin{table*}[h]
  \centering
  \caption{Inference deployment results with vLLM on GSM8K for  Llama‑3.2‑1B‑Instruct/Llama‑3.1‑8B‑Instruct (left) and Llama‑3.1‑8B‑Instruct/Llama‑3.1‑70B‑Instruct (right). Speed is measured in tokens/s, and \textit{Speedup is reported relative to standard speculative decoding (i.e., Vanilla SP)}.
  }
  \scalebox{0.65}{
    \begin{tabular}{c|c|ccccc}
    \toprule[1.5pt]
    Method & Metric & \multicolumn{5}{c}{Llama-3.2-1B-Instruct \& Llama-3.1-8B-Instruct} \\
    \midrule[1pt]
    \multirow{4}{*}{{AutoJudge}} 
          & Threshold & 0.01  & 0.07  & 0.09  & 0.14  & 0.22  \\
          & Acc       & 82.50\% & 80.29\% & 79.39\% & 78.40\% & 75.59\% \\
          & Speed     & 39.11 & 45.06 & 47.12 & 48.87 & 53.70 \\
          & Speedup   & 1.02$\times$ & 1.17$\times$ & 1.23$\times$ & 1.27$\times$ & 1.40$\times$ \\
    \midrule
    \multirow{4}{*}{{KL}} 
          & Threshold & 0.10  & 0.30  & 0.40  & 0.50  & 0.90  \\
          & Acc       & 82.77\% & 80.96\% & 81.07\% & 79.46\% & 77.13\% \\
          & Speed     & 39.58 & 45.35 & 48.35 & 50.96 & 61.45 \\
          & Speedup   & 1.03$\times$ & 1.18$\times$ & 1.26$\times$ & 1.33$\times$ & 1.60$\times$ \\
    \midrule
    \multirow{2}{*}{{Vanilla SP}} 
          & Acc   & \multicolumn{5}{c}{82.50\%} \\
          & Speed & \multicolumn{5}{c}{38.40} \\
    \bottomrule[1.5pt]
    \end{tabular}%
  }
  \hspace{0pt}
  \scalebox{0.65}{
    \begin{tabular}{c|c|ccccc}
    \toprule[1.5pt]
    Method & Metric & \multicolumn{5}{c}{Llama-3.1-8B-Instruct \& Llama-3.1-70B-Instruct} \\
    \midrule[1pt]
    \multirow{4}{*}{{AutoJudge}} 
          & Thr   & 0.04  & 0.07  & 0.09  & 0.14  & 0.22  \\
          & Acc   & 92.77\% & 92.54\% & 92.09\% & 91.08\% & 89.96\% \\
          & Speed & 27.34 & 29.55 & 30.98 & 32.57 & 35.07 \\
          & Speedup & 1.13$\times$ & 1.22$\times$ & 1.28$\times$ & 1.35$\times$ & 1.45$\times$ \\
    \midrule
    \multirow{4}{*}{{KL}} 
          & Thr   & 0.20  & 0.40  & 0.50  & 0.90  & 1.40  \\
          & Acc   & 92.69\% & 91.94\% & 91.88\% & 91.28\% & 90.27\% \\
          & Speed & 27.21 & 30.10 & 31.49 & 33.21 & 35.10 \\
          & Speedup & 1.12$\times$ & 1.24$\times$ & 1.30$\times$ & 1.37$\times$ & 1.45$\times$ \\
    \midrule

    \multirow{2}{*}{{Vanilla SP}} 
          & Acc   & \multicolumn{5}{c}{93.00\%} \\
          & Speed & \multicolumn{5}{c}{24.19} \\
    \bottomrule[1.5pt]
    \end{tabular}%
    }
  \label{ssy1122:vllm_results}
\end{table*}

\subsection{Main Results}

\paragraph{Superior Efficiency-Accuracy Trade-off.}
We first compare methods on GSM8K and LiveCodeBench (Figures \ref{ssy1122:performace_gsm8k}-\ref{ssy1122:performace_lcb}). The training-free KL thresholding exhibits highly consistent trends with the training-based AutoJudge, often outperforming it slightly.
Specifically, on the Llama-1B/8B pair for GSM8K, KL thresholding improves MAT by \textbf{63.5\%} (10.61$\rightarrow$17.35) over Vanilla SP with only a $\sim$2\% accuracy drop. Similarly, on the 8B/70B pair, MAT increases by \textbf{128.8\%} (13.35$\rightarrow$30.55) with a negligible $\sim$1\% accuracy decrease.
In contrast, baselines like Top-K and entropy-based strategies suffer from a much sharper accuracy degradation as MAT increases.
A similar MAT--accuracy trade-off is also observed on LiveCodeBench, where code generation requires strict precision. KL thresholding improves MAT by \textbf{104.0\%} (1B/8B, 8.83$\rightarrow$18.01) and \textbf{53.7\%} (8B/70B, 11.49$\rightarrow$17.66) with only minimal performance loss. Its performance is comparable to AutoJudge, and is even slightly better in some settings. In comparison, Top-K and entropy-based strategies often fail to preserve code correctness, frequently resulting in non-executable programs. The consistently similar trends between KL thresholding and AutoJudge
confirm that KL divergence serves as an effective proxy for learned
verification, aligning with our motivation
(Section~\ref{ssy1122:rethinking_supervision}) to replace trained
judges with model-intrinsic signals.

\vspace{3.8pt}\noindent\textbf{Robustness to Domain Shift.}
To evaluate generalization, we test on MATH-500-Hard and MMLU-Pro benchmarks (Figures \ref{ssy1122:performace_math500}-\ref{ssy1122:performace_mmlupro}) using the AutoJudge classifier trained only on GSM8K, as domain-specific supervision is often unavailable.
The results indicate that KL thresholding demonstrates superior robustness compared to AutoJudge.
Specifically, on MATH-500-Hard, KL thresholding improves MAT from 13.93 to 32.02 (1B/8B) with a negligible $<$1\% performance drop, and also increases MAT from 13.18 to 26.67 (8B/70B) with about a 1.5\% performance drop. Conversely, the GSM8K-trained AutoJudge classifier becomes brittle under this domain shift, showing clear accuracy degradation comparable to the weaker Top-K baseline.
This observation confirms that simple, training-free distributional statistics serve as a more reliable proxy for verification signals than classifiers trained on limited domains.

\vspace{3.8pt}\noindent\textbf{Real-world Speedup in vLLM.}
We integrate KL thresholding into vLLM to benchmark real-world end-to-end latency (as shown in Table~\ref{ssy1122:vllm_results}).
Consistent with our previous MAT-Accuracy findings, KL thresholding effectively matches the speed-quality trade-off of the computationally expensive AutoJudge. It delivers a \textbf{1.3$\times$ speedup} over the Vanilla SP baseline while simultaneously keeping accuracy within a relatively small 1\%--2\% margin.
This reinforces that costly critical-token mining is unnecessary for practical inference acceleration.


\begin{figure}[h]
    \centering
    \includegraphics[width=0.99\linewidth]{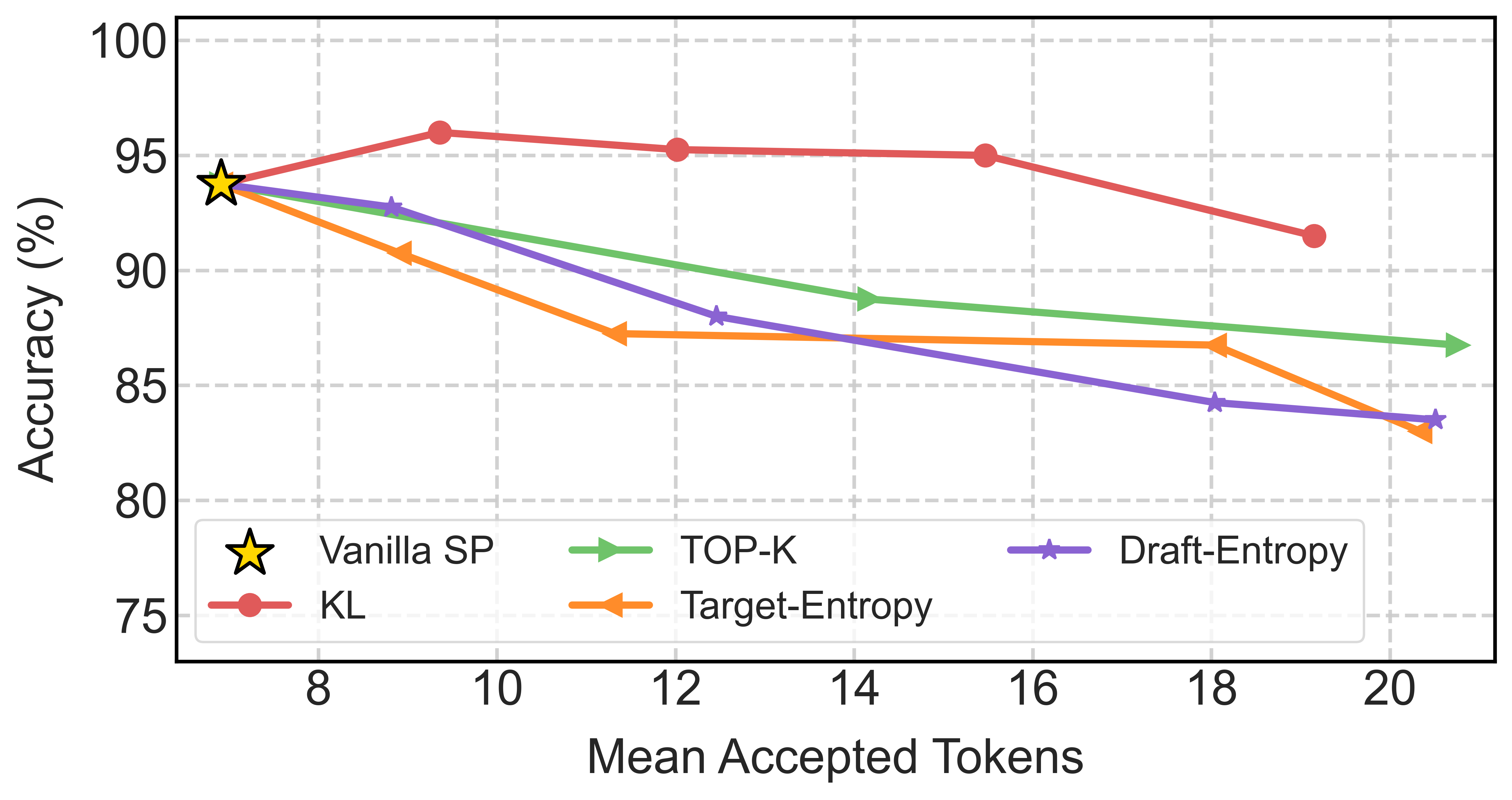}
    \caption{Accuracy and MAT on GSM8K for Qwen3-0.6B/Qwen3-8B.}
    \label{ssy1122:performace_qwen3}
\end{figure}


\vspace{3.8pt}\noindent\textbf{Efficiency in Long-Chain Reasoning.}
AutoJudge's reliance on critical-token mining is particularly prohibitive for reasoning models, whose long trajectories exhibit frequent draft-target divergence.
We empirically evaluate KL thresholding on the Qwen3-0.6B/Qwen3-8B pair (Figure~\ref{ssy1122:performace_qwen3}).
Remarkably, KL thresholding achieves a \textbf{2.77$\times$ increase in MAT} (6.91$\rightarrow$19.15) relative to Vanilla SP with only a marginal $\sim$2\% performance drop.
This evidence indicates that the KL thresholding effectively handles the complex divergence patterns of reasoning models without the need for intricate and expensive supervision mining.

\begin{table}
\centering
\caption{Output quality and MAT under full-vocab and Top-K KL thresholding. Evaluated on GSM8K with Llama-3.2-1B-Instruct / Llama-3.1-8B-Instruct.}
\label{ssy1122:topk_vocab}
\resizebox{\linewidth}{!}{%
\begin{tabular}{c cc cc cc}
\toprule[1.5pt]
\multirow{2}{*}{\textbf{Threshold}} & \multicolumn{2}{c}{\textbf{Full-vocab KL}} & \multicolumn{2}{c}{\textbf{Top-20}} & \multicolumn{2}{c}{\textbf{Top-200}} \\
\cmidrule(lr){2-3} \cmidrule(lr){4-5} \cmidrule(lr){6-7}
 & Acc & MAT & Acc & MAT & ACC & MAT \\
\midrule
Vanilla SP & 87.75\% & 10.61 & 87.75\% & 10.61 & 87.75\% & 10.61 \\
0.3 & 87.50\% & 13.33 & 86.00\%\,{\color{red}($\downarrow$1.50\%)} & 13.45 & 86.38\%\,{\color{red}($\downarrow$1.12\%)} & 13.35 \\
0.6 & 85.50\% & 17.35 & 85.62\% & 17.40 & 86.05\% & 16.72 \\
0.7 & 83.88\% & 18.94 & 84.00\% & 18.72 & 84.30\% & 18.90 \\
0.8 & 83.12\% & 19.65 & 83.62\% & 19.93 & 83.49\% & 19.91 \\
\bottomrule[1.5pt]
\end{tabular}%
}
\end{table}

\vspace{3.8pt}\noindent\textbf{Top-K Vocabulary Truncation.}
A natural question is whether restricting KL computation to only the Top-K tokens can further improve efficiency.
Perhaps counter-intuitively, Top-K truncation does not yield wall-clock speedups, as irregular memory access from selection and index-gathering offsets the theoretical complexity reduction (see Appendix~\ref{appendix:topk_vocab}).
More critically, Top-K truncation degrades acceptance quality in the low-threshold regime.
As shown in Table~\ref{ssy1122:topk_vocab}, at threshold 0.3, Top-20 and Top-200 incur \textbf{1.50\%} and \textbf{1.12\%} accuracy drops respectively relative to full-vocab KL, because a low threshold places the decision boundary in a high-density region of the KL distribution where truncation-induced estimation error flips a significant number of accept/reject decisions.

\begin{figure}[htbp]
    \centering
    \includegraphics[width=1.03\linewidth]{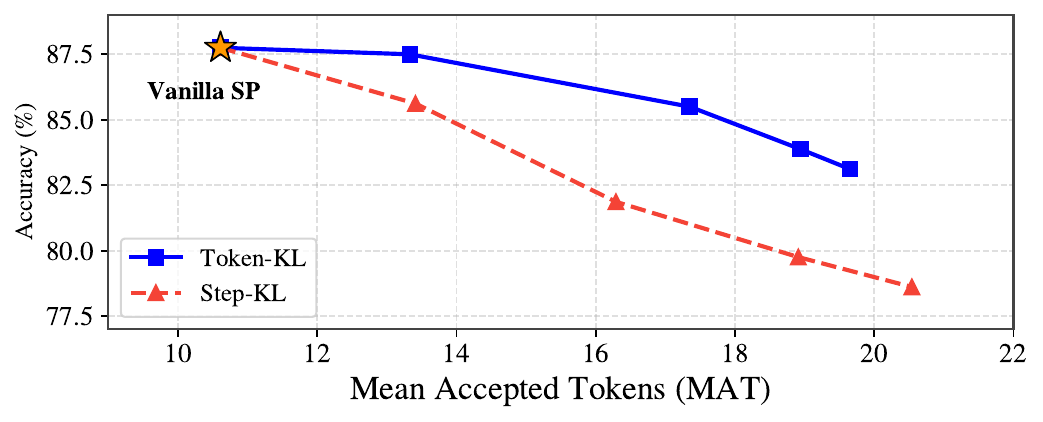}
    \caption{Accuracy--speedup trade-off of Token-KL vs.\ Step-KL on GSM8K (Llama-3.2-1B-Instruct / Llama-3.1-8B-Instruct).}
    \label{fig:step-vs-token}
\end{figure}

\vspace{3.8pt}\noindent\textbf{Step-Level KL Divergence Extension.} 
We investigate whether KL divergence can be applied at the \textit{step level}, i.e., aggregating token-level KL values within each reasoning step and making accept/reject decisions per step. Specifically, we segment the draft window into reasoning steps delimited by newline characters, compute the mean KL divergence within each step, and reject all tokens in a step if its aggregated KL exceeds the threshold. We evaluate on GSM8K using Llama-3.2-1B-Instruct/Llama-3.1-8B-Instruct. 
Results are shown in Figure~\ref{fig:step-vs-token}. At comparable MAT levels, Step-KL consistently underperforms Token-KL by approximately 2--5\% in accuracy (e.g., 85.62\% vs.\ 87.50\% at MAT$\approx$13; 78.62\% vs.\ 83.12\% at MAT$\approx$20). We attribute this degradation to three factors: \textit{(1) Signal dilution under aggregation}: the high KL signal from a few critical ``logic-pivoting'' tokens is diluted by non-critical tokens, causing entire steps containing errors to be incorrectly accepted; \textit{(2) Collateral rejection of correct tokens}: when a step exceeds the threshold due to one or two high-KL tokens, \emph{all} tokens in that step are rejected, wasting useful draft tokens and inflating re-drafting iterations; \textit{(3) Variable and unreliable step granularity}: reasoning steps vary widely in length, making a single threshold difficult to calibrate robustly. These results confirm that token-level granularity is the appropriate design choice for our framework.

\section{Related Work}
\label{sec:related_work}

\paragraph{Speculative Decoding.}
To accelerate LLM inference, speculative decoding employs a ``draft-then-verify'' principle~\cite{DBLP:arxiv/chenborgeaud23,DBLP:conf/icml/leviathankalman23}, where a lightweight draft model rapidly generates candidate tokens that are subsequently verified efficiently in parallel by the larger target model. By accepting multiple verified tokens per target-model forward pass, this paradigm reduces decoding latency while preserving the target distribution under exact verification. The key challenge is to generate drafts that are inexpensive enough to provide speedup while remaining sufficiently aligned with the target model to pass verification.
To further enhance drafting efficiency and avoid separate draft models, \textit{training-based} approaches have been proposed to modify the target model's architecture. Prominent examples include Medusa~\cite{DBLP:arxiv/caili24} and the EAGLE series~\cite{DBLP:conf/icml/liwei24a, DBLP:conf/emnlp/liwei24b, DBLP:arxiv/liwei25}. These methods improve proposal quality through model-specific adaptation, but typically require additional training or architectural changes. Their effectiveness therefore depends not only on the verification rule, but also on how well the auxiliary proposal mechanism is adapted to the target model.
Conversely, \textit{training-free} methods eliminate the need for additional parameters by exploiting the inherent redundancy in text. These strategies rely on retrieval or pattern matching, such as the N-gram based Lookahead decoding~\cite{DBLP:arxiv/fubailis24, DBLP:conf/kdd/zhaoxie24}, and retrieval-assisted generation like REST~\cite{DBLP:conf/naacl/hezhong24}, Token Recycling~\cite{DBLP:conf/acl/luowang25} and Suffix Decoding~\cite{DBLP:arxiv/oliarojia24,DBLP:conf/acl/huwang25}. They are broadly deployable because they avoid retraining, making them attractive in settings where model modification is costly or infeasible, yet they mainly improve candidate generation rather than reconsidering the verification criterion itself. Other applications of speculative decoding are discussed elsewhere~\cite{DBLP:conf/acl/xiayang24,DBLP:arxiv/SunLi26}.

\vspace{3.8pt}\noindent\textbf{Judge Decoding.}
Standard speculative decoding suffers from a verification criterion that rejects semantically equivalent but lexically distinct tokens. This is common with equivalent reasoning expressions.
To address this, Judge Decoding~\cite{DBLP:conf/iclr/BachmannAnagnostidis25} relaxes exact matching by employing a judge model to evaluate the semantic validity of draft tokens, drawing inspiration from the LLM-as-a-judge~\cite{DBLP:arxiv/GuJiang25,DBLP:conf/acl/BavarescoBernardi25}. This shifts verification from token-level identity toward semantic acceptability, allowing harmless surface-form variations to be accepted when they preserve the intended meaning.
Subsequent research focuses on automating the identification of ``critical tokens'', namely divergences that genuinely harm performance. AutoJudge~\cite{DBLP:conf/nips/GaripovVelikonivtsev25} automates data collection by pinpointing critical tokens that degrade model accuracy. 
SelfJudge~\cite{DBLP:journals/corr/abs-2510-02329} further refines this detection by calculating a semantic preservation score over the subsequent $N$ tokens to determine criticality. These methods reduce the burden of direct human annotation, but still require an external procedure to define and collect supervision for the judge.
Recent works introduce more granular verification based on utility and reasoning paths. 
Pivot-Aware Speculative Decoding~\cite{DBLP:journals/corr/abs-2511-00351} assesses token criticality by estimating the probability of generating correct answers via multiple sampled future paths. 
Similarly, R2R~\cite{DBLP:journals/corr/abs-2505-21600} distinguishes benign variations from logical errors by generating short continuations from both draft and target tokens and checking for semantic divergence. Overall, these studies highlight the importance of criticality-aware verification, while motivating simpler signals that can capture token importance without additional judge training.

\section{Conclusion}
\label{sec:conclusion}

This work bridges the gap between learned judge decoding and model-intrinsic statistics. We demonstrate that the costly ``criticality'' signal used in prior work is largely redundant to the distributional disagreement between models. By replacing supervised classifiers with a simple, training-free KL threshold, we achieve comparable or superior efficiency. Furthermore, our theoretical analysis unifies these seemingly distinct approaches, proving that KL screening and linear classifiers operate on the same underlying logit primitives. Our experiments validate that efficient judge decoding does not require complex supervision, offering a streamlined path for future deployment.

\clearpage
\section*{Limitations}
\label{sec:limitations}
We acknowledge some limitations in our work. First, due to computational constraints, we have not yet evaluated our approach on extremely large-scale models such as Llama-3.1-405B-Instruct, which would be essential for comprehensively validating scalability and real-world applicability. Second, the KL divergence computation introduces additional computational overhead during inference. We have not implemented highly optimized custom Triton kernels to optimize memory access and enable operator fusion, which could significantly reduce this cost. As a result, our current implementation may bottleneck the potential speedups achievable in modern high-performance inference frameworks like vLLM.

\section*{Ethical Considerations}
\label{sec:ethical_considerations}
The datasets utilized in this work are derived solely from publicly accessible benchmarks, ensuring that no sensitive or private information is included. Moreover, all LLMs and baseline methods employed in our experiments are also available to the public. We have taken care to acknowledge the original authors by properly citing their work. Regarding the preparation of this manuscript, we acknowledge the use of LLMs solely for language polishing and generating minor graphical icons for the framework diagram. They did not contribute to the research design, data analysis, or scientific claims presented in this paper. All ideas, analyses, and conclusions remain the sole and original work of the authors.

\section*{Acknowledgements}
This work is supported by the Early Career Scheme (No.CityU 21219323) and the General Research Fund (No.CityU 11220324) of the University Grants Committee (UGC), the NSFC Young Scientists Fund (No.9240127), and the Donation for Research Projects (No.9229216).

\bibliography{custom}

\clearpage
\appendix
\renewcommand{\thesection}{\Alph{section}}
\section{Appendix}\label{sec:appendix}

\subsection{Proof of Theorem \ref{ssy1122:structural_correspondence}}
\label{sec:a0_setup}
To facilitate the derivation, we first review the notation and setup. Let the target and draft logits be $z_t := W_t h_t + b_t$ and $z_d := W_d h_d + b_d$ respectively, with corresponding probability distributions $P_{t} := \mathrm{softmax}(z_{t})$ and $P_{d} := \mathrm{softmax}(z_{d})$. We define the concatenated hidden representation $x := [h_t; h_d] \in \mathbb{R}^{2d}$ and the logit difference vector $\delta := z_t - z_d$. Consequently, $\delta$ is affine in $x$ and is given by $\delta = Mx + c$, where $M := [W_t, -W_d]$ and $c := b_t - b_d$. For any vocabulary indices $i, j \in \{1, \dots, V\}$, the pairwise logit-gap difference is defined as $\Delta_{ij}(x) := (z_t(i) - z_t(j)) - (z_d(i) - z_d(j))$, which yields the $\Delta_{ij}(x) = \delta_i - \delta_j$. This linear primitive $\Delta_{ij}(x)$ serves as the basis for our analysis.

\begin{proof}[\textbf{Proof Sketch of Theorem \ref{ssy1122:structural_correspondence}}] We ground the analysis in the affine primitives $\Delta_{ij}(x)$. By expanding the KL divergence via the Fisher information metric (\S\ref{sec:a1_kl_bregman_fisher}), we decompose the divergence into a weighted quadratic aggregation of these primitives (\S\ref{sec:a2_pairwise_kl}). We subsequently demonstrate that discrete ranking inconsistencies necessitate boundary-crossing deviations (\S\ref{sec:a3_topk_mismatch}), thereby enforcing a quantitative lower bound on the divergence (\S\ref{sec:a4_kl_lower_bound}). Finally, we prove that the linear classifier's decision surface forms a linear superposition of this same primitive basis (\S\ref{sec:a5_linear_classifier_alignment}), confirming that both mechanisms regulate the same logit-space deviations.
\end{proof}
\subsubsection{KL as a Bregman divergence and Fisher second-order approximation}
\label{sec:a1_kl_bregman_fisher}

We define the log-sum-exp potential $A(z) := \log \sum_{i=1}^V e^{z_i}$, which satisfies $\nabla A(z) = p$ and $\nabla^2 A(z) = \mathrm{diag}(p) - pp^\top$ for $p = \mathrm{softmax}(z)$.

\begin{lemma}[KL as a Bregman divergence]
\label{lem:a1_bregman}
Let $p=\mathrm{softmax}(z)$ and $q=\mathrm{softmax}(z')$. Then $D_{\mathrm{KL}}(p\Vert q)=A(z')-A(z)-\nabla A(z)^\top(z'-z)$.
\end{lemma}

\begin{proof}
By the log-softmax identities $\log p_i = z_i - A(z)$ and $\log q_i = z'_i - A(z')$, we have:
\begin{align}
D_{\mathrm{KL}}&(p \Vert q)= \sum_i p_i \left( \log p_i - \log q_i \right) \nonumber \\
&= \sum_i p_i \left[ (z_i - A(z)) - (z'_i - A(z')) \right] \nonumber \\
&= (A(z') - A(z)) \sum_i p_i - \sum_i p_i (z'_i - z_i)  \nonumber \\
&= A(z') - A(z) - \nabla A(z)^\top (z' - z),
\label{eq:a1_summation_step}
\end{align}
where the last step follows from $\sum_i p_i = 1$ and the fact that $\nabla A(z) = p$.
\end{proof}

\begin{lemma}[Second-order expansion around $z_d$]
\label{lem:a1_second_order}
Let $P_d=\mathrm{softmax}(z_d)$, $P_t=\mathrm{softmax}(z_t)$ and write $z_t=z_d+\delta$.
Conditioned on accepted prefixes in speculative sampling, the two models are close so that $\|\delta\|$ is small (large discrepancies are rejected and do not contribute).
Then $D_{\mathrm{KL}}(P_t\|P_d)=\frac12\,\delta^\top F(P_d)\,\delta + O(\|\delta\|^3)$, where $F(P_d)=\nabla^2 A(z_d)=\mathrm{diag}(P_d)-P_dP_d^\top$.
\end{lemma}
\begin{proof}
Applying Lemma \ref{lem:a1_bregman} with $p=P_t$ and $q=P_d$, and noting $z_d - z_t = -\delta$, we have the identity $D_{\mathrm{KL}}(P_t \| P_d) = A(z_d) - A(z_t) + P_t^\top \delta$.
We expand $A(z_t)$ to the second order as $A(z_t) = A(z_d) + P_d^\top \delta + \frac{1}{2} \delta^\top F(P_d) \delta + O(\|\delta\|^3)$.
Similarly, expanding $P_t$ to the first order gives $P_t = P_d + F(P_d) \delta + O(\|\delta\|^2)$, which implies $P_t^\top \delta = P_d^\top \delta + \delta^\top F(P_d) \delta + O(\|\delta\|^3)$.
Substituting these expressions into the KL identity leads to the cancellation of linear terms, yielding $D_{\mathrm{KL}}(P_t \| P_d) = \frac{1}{2} \delta^\top F(P_d) \delta + O(\|\delta\|^3)$.
\end{proof}

\begin{corollary}[KL's quadratic form controlled by a linear map of the concatenated state]
\label{cor:a1_quadratic_x}
Recall that $\delta=z_t-z_d$ and $\delta = Mx + c$ with $x=[h_t;h_d]$.
Combining the second-order expansion in Lemma~\ref{lem:a1_second_order} with $\delta = Mx+c$, we obtain
\begin{equation}
D_{\mathrm{KL}}(P_t\Vert P_d)
\approx
\frac12\,(Mx+c)^\top F(P_d)\,(Mx+c).
\end{equation}
\end{corollary}

\subsubsection{Pairwise form: KL as a Fisher-weighted sum of primitives}
\label{sec:a2_pairwise_kl}

\begin{lemma}[Pairwise decomposition of the Fisher quadratic form]
\label{lem:a2_pairwise_decomp}
Let $p\in\mathbb{R}^V$ satisfy $p_i\ge 0$ and $\sum_{i=1}^V p_i=1$, and let
$F(p)=\mathrm{diag}(p)-pp^\top$. Then for any $\delta\in\mathbb{R}^V$,
\begin{equation}
\frac12\,\delta^\top F(p)\,\delta
=
\frac14\sum_{i=1}^V\sum_{j=1}^V
p_i p_j\,(\delta_i-\delta_j)^2.
\label{eq:a2_pairwise_fisher}
\end{equation}
\end{lemma}

\begin{proof}
Expand the left-hand side using $F(p)=\mathrm{diag}(p)-pp^\top$, we have
\begin{align}
\delta^\top F(p)\,\delta
&=\delta^\top \mathrm{diag}(p)\,\delta - \delta^\top (pp^\top)\delta \nonumber\\
&=\sum_{i=1}^V p_i \delta_i^2 - \bigl(p^\top\delta\bigr)^2\\
&=\sum_{i=1}^V p_i \delta_i^2 - \Bigl(\sum_{i=1}^V p_i\delta_i\Bigr)^2.
\label{eq:a2_fisher_expand}
\end{align}
Expand the right-side pairwise sum, we have
\begin{align}
&\sum_{i=1}^V\sum_{j=1}^V p_ip_j(\delta_i-\delta_j)^2=\sum_{i,j}p_ip_j(\delta_i^2+\delta_j^2-2\delta_i\delta_j) \nonumber\\
&=\sum_{i,j}p_ip_j\delta_i^2 + \sum_{i,j}p_ip_j\delta_j^2
-2\sum_{i,j}p_ip_j\delta_i\delta_j\nonumber\\
&=2\sum_i p_i\delta_i^2
-2\Bigl(\sum_i p_i\delta_i\Bigr)^2,
\label{eq:a2_pairwise_expand}
\end{align}
where the last equality follows from Eq. \eqref{eq:a2_fisher_expand}.
Dividing both sides by $4$ yields Eq. \eqref{eq:a2_pairwise_fisher}.
\end{proof}

\begin{corollary}[KL as quadratic aggregation of $\Delta_{ij}$]
\label{cor:a2_kl_quadratic_aggregation}
Combining Lemma~\ref{lem:a1_second_order} and Lemma~\ref{lem:a2_pairwise_decomp}
with $p=P_d$ and $\Delta_{ij}(x)=\delta_i-\delta_j$, we have
$D_{\mathrm{KL}}(P_t\Vert P_d)\approx\sum_{i=1}^V\sum_{j=1}^V\alpha_{ij}\,\Delta_{ij}(x)^2$, where $\alpha_{ij}:=\frac14\,P_d(i)\,P_d(j).$
\end{corollary}


\subsubsection{Top-K inconsistency and boundary-crossing primitive}
\label{sec:a3_topk_mismatch}

The discrete supervision signal used for classifier training is defined through the consistency between the target and draft models’ Top-K prediction sets.  
Since LLMs assign most of the probability mass to a small subset of tokens, agreement of the Top-K sets corresponds to similar generation behavior, whereas disagreement serves as an indicator of divergent outputs.  
A top‑$k$ inconsistency implies the existence of a pair of logits whose pairwise difference crosses the ranking boundaries, yielding a lower bound on the primitive $\Delta_{ij}(x)$. Let $T_k(x)$ and $D_k(x)$ denote the index sets of the largest $k$ logits under $z_t$ and $z_d$, respectively.  
Let $z_t^{(k)}$ and $z_d^{(k)}$ denote the $k$‑th largest logit values of $z_t$ and $z_d$. Denote the top‑$k$ margins by
$\gamma_t^{(k)} = z_t^{(k)} - z_t^{(k+1)}$ and
$\gamma_d^{(k)} = z_d^{(k)} - z_d^{(k+1)}$.

\begin{proposition}[Top‑$k$ inconsistency implies a boundary‑crossing pairwise gap]
\label{prop:a3_topk_mismatch_boundary_crossing}
If $T_k(x)\neq D_k(x)$, then there exist
$i\in T_k(x)\setminus D_k(x)$ and
$j\in D_k(x)\setminus T_k(x)$
such that
$\Delta_{ij}(x)\ge \gamma_t^{(k)}+\gamma_d^{(k)}$.
\end{proposition}

\begin{proof}
Since $T_k(x)\neq D_k(x)$, both $T_k(x)\setminus D_k(x)$ and $D_k(x)\setminus T_k(x)$ are non‑empty.  
Select any $i\in T_k(x)\setminus D_k(x)$ and $j\in D_k(x)\setminus T_k(x)$. For $z_t$, $i\in T_k(x)$ implies $z_t(i)\ge z_t^{(k)}$, and $j\notin T_k(x)$ implies $z_t(j)\le z_t^{(k+1)}$, we have
\begin{equation}
z_t(i)-z_t(j)\ge z_t^{(k)}-z_t^{(k+1)}=\gamma_t^{(k)}. 
\label{ssy1122:topk_ineq1}
\end{equation}

For $z_d$, $j\in D_k(x)$ implies $z_d(j)\ge z_d^{(k)}$, and $i\notin D_k(x)$ implies $z_d(i)\le z_d^{(k+1)}$, we have
\begin{equation}
z_d(i)-z_d(j)\le z_d^{(k+1)}-z_d^{(k)}=-\gamma_d^{(k)}.
\label{ssy1122:topk_ineq2}
\end{equation}

Combining the above inequalities Eq. \eqref{ssy1122:topk_ineq1} and Eq. \eqref{ssy1122:topk_ineq2} gives $\Delta_{ij}(x) \ge \gamma_t^{(k)}+\gamma_d^{(k)}$.
\end{proof}

Top‑$k$ inconsistency therefore ensures the existence of a boundary‑crossing primitive $(i,j)$ whose logit‑difference gap exceeds the sum of the target and draft margins.  
This provides the link between the discrete supervision condition and a continuous deviation in the primitive space, which is used in Appendix~\ref{sec:a4_kl_lower_bound} to establish a lower bound on the KL divergence.

\subsubsection{From boundary-crossing primitives to a KL lower bound}
\label{sec:a4_kl_lower_bound}

Proposition~\ref{prop:a3_topk_mismatch_boundary_crossing} established that 
a top‑$k$ mismatch implies the existence of a boundary‑crossing primitive $(i,j)$ 
whose pairwise logit‑difference satisfies 
$|\Delta_{ij}(x)| \ge \gamma_t^{(k)} + \gamma_d^{(k)}$. 
Meanwhile, Corollary~\ref{cor:a2_kl_quadratic_aggregation} expresses the 
KL as a Fisher‑weighted quadratic aggregation 
over all such primitives. 
Combining these two facts shows that the presence of a boundary‑crossing 
primitive induces a quantitative lower bound on the overall KL divergence.

\begin{proposition}[Boundary‑crossing primitive implies a KL lower bound]
\label{prop:a4_primitive_to_kl}
If for some index pair $(i,j)$ we have
$|\Delta_{ij}(x)| \ge \gamma_t^{(k)} + \gamma_d^{(k)}$, then based on the quadratic expansion in Lemma~\ref{lem:a1_second_order},
$D_{\mathrm{KL}}(P_t\Vert P_d)
\ \gtrsim\
\alpha_{ij}\bigl(\gamma_t^{(k)} + \gamma_d^{(k)}\bigr)^{2}$, where $\alpha_{ij}=\tfrac14P_d(i)P_d(j)$.
\end{proposition}

\begin{proof}
From Corollary~\ref{cor:a2_kl_quadratic_aggregation}, dropping the non-negative terms for pairs other than $(i,j)$ yields $D_{\mathrm{KL}}(P_t\Vert P_d) \approx \sum_{u,v} \alpha_{uv}\Delta_{uv}(x)^2 \ge \alpha_{ij}\Delta_{ij}(x)^2 \ge \alpha_{ij}(\gamma_t^{(k)} + \gamma_d^{(k)})^2$.
\end{proof}
This yields a lower bound on the KL divergence,
showing that any top‑$k$ inconsistency enforces 
a separation between the two 
predictive distributions.

\subsubsection{Why a linear classifier trained on ``importance'' aligns with KL screening}
\label{sec:a5_linear_classifier_alignment}

The subsequent lemma writes $\Delta_{ij}(x)$ explicitly as an affine function of $x$.

\begin{lemma}[Affine form of the primitives]
\label{lem:a5_primitives_affine}
For any $i,j\in\{1,\dots,V\}$ there exist $a_{ij}\in\mathbb{R}^{2d}$ and
$\kappa_{ij}\in\mathbb{R}$ such that $\Delta_{ij}(x)=a_{ij}^\top x+\kappa_{ij}$.
\end{lemma}

\begin{proof}
Let $e_i\in\mathbb{R}^V$ denote the $i$-th standard basis vector. By definition, $z_t=W_t h_t+b_t$ and $z_d=W_d h_d+b_d$, hence $z_t(i)=e_i^\top(W_t h_t+b_t)$ and $z_d(i)=e_i^\top(W_d h_d+b_d)$. Therefore
$z_t(i)-z_t(j)=(e_i-e_j)^\top(W_t h_t+b_t)$ and $z_d(i)-z_d(j)=(e_i-e_j)^\top(W_d h_d+b_d)$.
Subtracting the two identities gives $\Delta_{ij}(x)=(e_i-e_j)^\top W_t h_t-(e_i-e_j)^\top W_d h_d+(e_i-e_j)^\top(b_t-b_d)$. Define $a_{ij}:=\bigl[\,W_t^\top(e_i-e_j);\,-W_d^\top(e_i-e_j)\,\bigr]$
and $\kappa_{ij}:=(e_i-e_j)^\top(b_t-b_d)$. Using $x=[h_t;h_d]$ yields $\Delta_{ij}(x)=a_{ij}^\top x+\kappa_{ij}$.
\end{proof}

\begin{proposition}[Representation of the classifier in the primitive basis]
\label{prop:a5_classifier_primitive_linear}
Let $\mathrm{Cls}(x)=\sigma(w^\top x+b)$ be a linear classifier trained on the concatenated representation $x$.
In LLMs, the vocabulary size $V$ is typically much larger than the hidden dimension $d$ ($V \gg d$).
Consequently, the set of primitive direction vectors $\{\,a_{ij}\mid i,j=1,\dots,V\,\}$ forms an overcomplete frame that effectively spans the feature space $\mathbb{R}^{2d}$.
The classifier’s weight vector $w$ can therefore be expressed as a linear combination of primitives $w = \sum_{i=1}^V \sum_{j=1}^V \beta_{ij}\, a_{ij}$ for a set of coefficients $\{\beta_{ij}\}$. Accordingly, the classifier is equivalent to $\mathrm{Cls}(x)=\sigma\!\left(\sum_{i=1}^V \sum_{j=1}^V \beta_{ij}\,\Delta_{ij}(x)+b'\right)$, for an adjusted bias term $b'$. This representation indicates that the decision surface can be well characterized by a linear combination of the primitives $\Delta_{ij}(x)$.
\end{proposition}

\begin{proof}
Because the set $\{\,a_{ij}\mid i,j=1,\dots,V\,\}$ effectively spans the feature space $\mathbb{R}^{2d}$,
the classifier weight vector $w$ admits the expansion $w = \sum_{i=1}^V \sum_{j=1}^V \beta_{ij}a_{ij}$. Substituting the affine form of each primitive from Lemma~\ref{lem:a5_primitives_affine}, we obtain
\begin{align}
w^\top x 
&= \Bigl(\sum_{i=1}^V \sum_{j=1}^V \beta_{ij}a_{ij}\Bigr)^\top x 
= \sum_{i=1}^V \sum_{j=1}^V \beta_{ij}(a_{ij}^\top x) \nonumber\\
&= \sum_{i=1}^V \sum_{j=1}^V \beta_{ij}\bigl(\Delta_{ij}(x) - \kappa_{ij}\bigr).
\end{align}
By defining the bias
$b' := b - \sum_{i=1}^V \sum_{j=1}^V \beta_{ij}\kappa_{ij}$,
the classifier can be rewritten in terms of $\Delta_{ij}(x)$ as $
w^\top x + b = \sum_{i=1}^V \sum_{j=1}^V \beta_{ij}\,\Delta_{ij}(x) + b'$, which yields $\mathrm{Cls}(x)=$$\sigma\!\left(\sum_{i=1}^V \sum_{j=1}^V \beta_{ij}\,\Delta_{ij}(x)+b'\right)$.
\end{proof}

Corollary~\ref{cor:a2_kl_quadratic_aggregation} provides the second-order primitive
expansion of the KL divergence. Together with Proposition~\ref{prop:a5_classifier_primitive_linear},
this yields the following formal juxtaposition in the primitive coordinate system
\[
\left\{
\begin{aligned}
&D_{\mathrm{KL}}(P_t\Vert P_d)\approx
\sum_{i=1}^V \sum_{j=1}^V \tfrac14\,P_d(i)P_d(j)\Delta_{ij}(x)^2,
\\[6pt]
&\mathrm{Cls}(x)
\ \equiv\
\sigma\!\left(\sum_{i=1}^V \sum_{j=1}^V \beta_{ij}\,\Delta_{ij}(x)+b'\right),
\end{aligned}
\right.
\]
The first line corresponds to a weighted quadratic aggregation of the primitives~$\Delta_{ij}(x)$,
whereas the second line represents a linear aggregation over the same primitive basis,
demonstrating that the classifier operates directly on the fundamental components of the KL divergence. Consequently, the structural correspondence asserted in Theorem~\ref{ssy1122:structural_correspondence} is formally established.

\begin{figure}[h]
    \centering
    \includegraphics[width=1.0\linewidth]{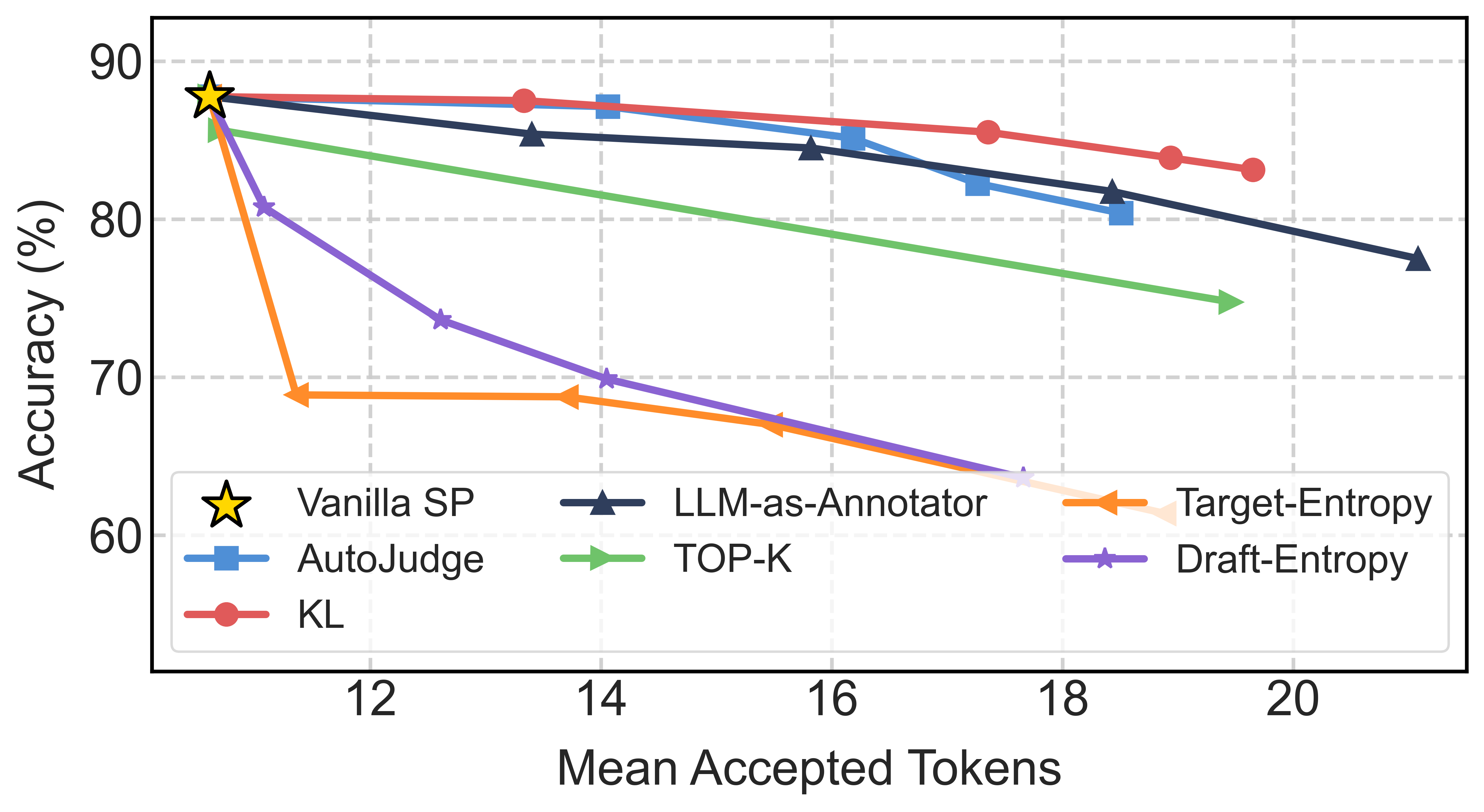}
    \caption{Accuracy and MAT on GSM8K for Llama-3.2-1B-Instruct and Llama-3.1-8B-Instruct (with Qwen3-Max as the annotator).}
    \label{ssy1122:performace_qwen3max}
\end{figure}

\begin{figure*}[!t]
  \centering
  \begin{tcolorbox}[aclprompt, title={Prompt for Critical-Token Annotation}]
\ttfamily\footnotesize
\textbf{[System Prompt]} 
\begin{lstlisting}[style=promptstyle]
    You are an expert annotator.
\end{lstlisting}
\textbf{[User Prompt]}
\begin{lstlisting}[style=promptstyle]
    You are an expert annotator. Your task is to identify the KEY REASONING ERROR tokens that 
    cause the wrong answer to be incorrect.

    INPUT:
        Question: [to be inserted]
        Correct Answer: [to be inserted]
        Wrong Answer: [to be inserted]

    INSTRUCTIONS:
        1. Identify WHERE the logical/reasoning ERROR occurs (not just textual differences).
        2. Mark ONLY the specific tokens that introduce the error.
        3. There may be MULTIPLE error locations - mark each separately.
        4. Ignore stylistic differences - focus on semantic/logical mistakes.
        5. An error is a token that causes incorrect reasoning, wrong calculation, or false 
        conclusion.

    OUTPUT FORMAT (STRICT):
        ANNOTATED_ANSWER: [exact wrong answer text with markers inserted]
    
    Use markers:
        - [ERROR_START] before the first token of each error
        - [ERROR_END] after the last token of each error
        - Keep ALL original text and spacing
        - Multiple errors = multiple [ERROR_START]...[ERROR_END] pairs
    
    EXAMPLE 1:
        Question: What is 5 + 3?
        Correct: 5 + 3 = 8
        Wrong: 5 + 3 = 9
        Output: ANNOTATED_ANSWER: 5 + 3 = [ERROR_START]9[ERROR_END]
    
    EXAMPLE 2:
        Question: If I hang 5 shirts outside and it takes them 5 hours to dry, how long would 
        it take to dry 30 shirts?
        Correct: It would still take 5 hours to dry 30 shirts, as drying is parallel.
        Wrong: It would take 30 hours to dry 30 shirts, as each shirt needs an additional hour 
        to dry.
        Output: ANNOTATED_ANSWER: It would take [ERROR_START]30 hours to dry 30 shirts, as each 
        shirt needs an additional hour to dry.[ERROR_END]
    
    EXAMPLE 3:
        Question: What is the capital of France and Germany?
        Correct: Paris and Berlin.
        Wrong: Paris is the capital of France and Munich is the capital of Germany.
        Output: ANNOTATED_ANSWER: Paris is the capital of France and [ERROR_START]Munich[ERROR_
        END] is the capital of Germany.
    
    Now process the input and output ONLY the ANNOTATED_ANSWER with error markers.
\end{lstlisting}
  \end{tcolorbox}
  \caption{Prompt for critical-token annotation, using Qwen3-Max as the LLM annotator.}
  \label{ssy1122:annot_prompt}
\end{figure*}

\begin{figure*}[!t]
\centering
\begin{tcolorbox}[aclprompt,title={GSM8K Example (Question ID: 5)},width=\textwidth]
\begin{lstlisting}[style=promptstyle]
(*@\textbf{[Question]}@*)
    James writes a 3-page letter to 2 different friends twice a week.  How many pages does he 
    write a year?

(*@\textbf{[Correct Answer]}@*)
    To find the number of pages James writes a year, we need to calculate the number of pages 
    he writes per week and then multiply that by the number of weeks in a year.

    James writes a 3-page letter to 2 different friends twice a week. So, he writes 2 x 2 = 4 
    letters per week.

    Since there are 3 pages in each letter, James writes 4 x 3 = 12 pages per week.

    There are 52 weeks in a year. So, James writes 12 x 52 = 624 pages per year.

    The final answer is 624.

(*@\textbf{[Annotated Wrong Answer]}@*)
    To find out how many pages James writes a year, we need to calculate the number of pages 
    he writes per week and then multiply it by the number of weeks in a year.

    James writes a 3-page letter to 2 different friends twice a week. This means he writes (*@\textcolor{red}{\textbf{[ER}}@*)
    (*@\textcolor{red}{\textbf{ROR\_START]}}@*)2 * 3 = 6 pages per week.(*@\textcolor{red}{\textbf{[ERROR\_END]}}@*)

    Since there are 52 weeks in a year, James writes 6 * 52 = 312 pages per year.

    The final answer is 312.
\end{lstlisting}
\end{tcolorbox}
\caption{Example of Qwen3-Max annotation on the GSM8K (Question ID: 5), with \textcolor{red}{\textbf{[ERROR$\_$START]}} and \textcolor{red}{\textbf{[ERROR$\_$END]}} marking critical tokens in the draft model’s erroneous solution.}
\label{ssy1122:labled_example_id5}
\end{figure*}

\subsection{Additional Experiment Details}
\label{ssy1122:additional_exp_details}
We conduct experiments on four benchmarks: GSM8K and MATH-500-Hard for math reasoning, MMLU-Pro for broad-coverage knowledge and reasoning, and LiveCodeBench for code generation. GSM8K consists of grade-school math word problems that require multi-step numerical reasoning. We additionally evaluate on MATH-500-Hard, which we construct by selecting all Level-5 problems from MATH-500, to introduce a math distribution that differs from GSM8K and is substantially more challenging. MMLU-Pro is a more difficult variant of MMLU spanning a wide range of subjects. LiveCodeBench focuses on algorithmic code generation with execution-based evaluation. To manage end-to-end decoding cost, we run evaluation on held-out subsets. For GSM8K, we evaluate on 400 examples from the test split, while AutoJudge training uses the full GSM8K training split. For MMLU-Pro, we evaluate on 100 examples. For LiveCodeBench, we partition the dataset into five folds, use four folds to train the AutoJudge classifier, and reserve the remaining fold as the shared test set for all methods.

\subsection{Potential of LLM Annotation}
\label{ssy1122:appendix_llm_annotation}

Given the limitations of labor-intensive manual annotation and the high computational overhead of heuristic mining, a natural alternative is to leverage an LLM as a surrogate annotator to provide supervision for judge training. In this appendix, we explore the feasibility of using Qwen3-Max\footnote{\url{https://qwen.ai/blog?id=qwen3-max}} as an automatic labeler to identify logic-pivoting tokens in erroneous generations. We construct an annotation set from the GSM8K training split by first filtering for instances where the target model produces a correct final answer while the draft model produces an incorrect one. From this subset, we then sample 2,000 examples. For each problem, we provide Qwen3-Max with the problem statement, the draft model’s erroneous output, and the target model’s verified correct answer, prompting it to highlight the specific words or tokens in the draft output that lead to the incorrect final answer (i.e., error-triggering or logic-pivoting tokens). The detailed prompt is shown in Figure \ref{ssy1122:annot_prompt}. We then use these token-level annotations to train a lightweight linear classifier following the same training methodology as AutoJudge. In total, the annotation cost is approximately 2M tokens (about \$10).\\
\indent The results in Figure \ref{ssy1122:performace_qwen3max} indicate that the judge trained on Qwen3-Max-labeled data consistently underperforms the judge trained with AutoJudge-mined labels in the low-performance-loss regime (upper-left region of the curve). In other words, although LLM annotation is substantially cheaper than counterfactual rollout-based mining, its supervision does not yield a more accurate token accept/reject rule when decoding requires high precision (i.e., minimal quality degradation). This gap suggests a key limitation of LLM-based annotation in our setting. This issue is also mentioned in \citet{DBLP:conf/iclr/BachmannAnagnostidis25}: using LLMs for this purpose proved to be too imprecise. Qwen3-Max may highlight tokens that look suspicious from a general semantic perspective, but these tokens are not always the ones that flip the final outcome of the generation. As a result, the labels can be informative but less well aligned with what judge decoding needs, especially in the high-precision regime where a small amount of label noise can noticeably worsen the quality--speed trade-off. Please refer to Figures \ref{ssy1122:labled_example_id5}-\ref{ssy1122:labled_example_id68} for annotation examples.

\begin{figure*}[!t]
\centering
\begin{tcolorbox}[aclprompt,title={GSM8K Example (Question ID: 68)},width=\textwidth]
\begin{lstlisting}[style=promptstyle]
(*@\textbf{[Question]}@*)
    Jesse and Mia are competing in a week long race. They have one week to run 30 miles. On
    the first three days Jesse averages (2/3) of a mile. On day four she runs 10 miles. Mia
    averages 3 miles a day over the first 4 days. What is the average of their average that
    they have to run over the final three days?

(*@\textbf{[Correct Answer]}@*)
    To find the average of their average that they have to run over the final three days, we
    need to first find their average for the entire race and then subtract the average they
    have to run over the first four days.

    Jesse's average for the first three days is (2/3) miles per day. So, over the first three 
    days, she runs a total of (2/3) * 3 = 2 miles.

    On day four, Jesse runs 10 miles. So, her total distance for the first four days is 2 + 1
    0 = 12 miles.

    Mia's average for the first four days is 3 miles per day. So, over the first four days,
    she runs a total of 3 * 4 = 12 miles.

    The total distance they have to run over the entire week is 30 miles. So, the average di-
    stance they have to run over the final three days is 30 - 12 = 18 miles.

    The final answer is 6.

(*@\textbf{[Annotated Wrong Answer]}@*)
    To find the average distance Jesse and Mia need to run over the final three days, we first 
    need to calculate their total distances for the first four days.

    Jesse's total distance for the first four days:
    - Day 1: (2/3) miles
    - Day 2: (2/3) miles
    - Day 3: (2/3) miles
    - Day 4: 10 miles
    Total distance for the first four days = (2/3) + (2/3) + (2/3) + 10 = (*@\textcolor{red}{\textbf{[ERROR\_START]}}@*)8 + 10 
    = 18 miles(*@\textcolor{red}{\textbf{[ERROR\_END]}}@*)

    Mia's total distance for the first four days:
    - Day 1: 3 miles
    - Day 2: 3 miles
    - Day 3: 3 miles
    - Day 4: (*@\textcolor{red}{\textbf{[ERROR\_START]}}@*)10 miles(*@\textcolor{red}{\textbf{[ERROR\_END]}}@*)
    Total distance for the first four days = 3 + 3 + 3 + 10 = (*@\textcolor{red}{\textbf{[ERROR\_START]}}@*)19 miles(*@\textcolor{red}{\textbf{[ERROR\_END]}}@*)
    
    Now, we need to find the average distance each of them needs to run over the final 3 days. 

    Jesse's average distance for the final three days:
    Total distance for the first four days = (*@\textcolor{red}{\textbf{[ERROR\_START]}}@*)18 miles(*@\textcolor{red}{\textbf{[ERROR\_END]}}@*)
    Number of days = 3
    Mean distance = Total distance / Number of days = (*@\textcolor{red}{\textbf{[ERROR\_START]}}@*)18 / 3 = 6 miles(*@\textcolor{red}{\textbf{[ERROR\_
    END]}}@*)
    
    Mia's average distance for the final three days:
    Total distance for the first four days = (*@\textcolor{red}{\textbf{[ERROR\_START]}}@*)19 miles(*@\textcolor{red}{\textbf{[ERROR\_END]}}@*)
    Number of days = 3
    Average distance = Total distance / Number of days = (*@\textcolor{red}{\textbf{[ERROR\_START]}}@*)19 / 3 = 6.33 miles(*@\textcolor{red}{\textbf{[ERR}}@*)
    (*@\textcolor{red}{\textbf{OR\_END]}}@*)
    
    The final answer is: $\boxed{6.33}$
    
\end{lstlisting}
\end{tcolorbox}
\caption{Example of Qwen3-Max annotation on the GSM8K (Question ID: 68), with \textcolor{red}{\textbf{[ERROR$\_$START]}} and \textcolor{red}{\textbf{[ERROR$\_$END]}} marking critical tokens in the draft model’s erroneous solution.}
\label{ssy1122:labled_example_id68}
\end{figure*}

\subsection{Detailed Decoding Algorithm}
\label{appendix:algorithm}

As discussed in Section~\ref{ssy1122:revisit_judge_decoding}, both
supervision paradigms for judge decoding bottleneck on the cost and
reliability of token-level labels. Our key observation is that the KL
divergence $D_{\mathrm{KL}}(P_T \| P_D)$ already captures the
semantic discrepancy a trained judge attempts to learn, yielding an
annotation-free acceptance criterion with no rollouts and no
additional parameters. The resulting procedure is detailed in Algorithm~\ref{ssy1122:judge_decoding_alg}.

\begin{algorithm}[h]
\caption{Annotation-Free Judge Decoding}
\label{ssy1122:judge_decoding_alg}
\begin{algorithmic}[1]

\REQUIRE Prompt $\mathbf{x}$, target model $M_T$, draft model $M_D$,
         draft length $\gamma$, KL threshold $\tau$,
         confidence threshold $\tau_c$

\ENSURE  Generated sequence $\mathbf{x}$

\REPEAT

\STATE {\color{blue}\texttt{// Stage 1: Draft}}
\STATE $M_D$ drafts $\gamma$ tokens
       $d_1, \dots, d_\gamma$, storing draft logits
       $\mathbf{z}_D^{(i)}$ at each position $i$.

\STATE {\color{blue}\texttt{// Stage 2: Verify}}
\STATE $M_T$ verifies $d_1, \dots, d_\gamma$ in one forward pass,
       yielding logits $\mathbf{z}_T^{(i)}$ and greedy
       predictions $t_i = \arg\max\, \mathbf{z}_T^{(i)}$ for
       $i = 1, \dots, \gamma$.

\STATE {\color{blue}\texttt{// Stage 3: Accept / Reject}}
\FOR{$i = 1$ \TO $\gamma$}
    \STATE $P_T \leftarrow \mathrm{softmax}\!\left(\mathbf{z}_T^{(i)}\right)$
    \STATE $P_D \leftarrow \mathrm{softmax}\!\left(\mathbf{z}_D^{(i)}\right)$
    \STATE $\mathrm{KL}_i \leftarrow D_{\mathrm{KL}}(P_T \| P_D)$

    \IF{$\max(P_T) > \tau_c$}
        \STATE $\textit{accept}_i \leftarrow (d_i = t_i)$
    \ELSE
        \STATE $\textit{accept}_i \leftarrow (d_i = t_i)
               \;\lor\; (\mathrm{KL}_i < \tau)$
    \ENDIF

    \IF{\NOT $\textit{accept}_i$}
        \STATE Replace $d_i$ with $t_i$; truncate after
               position $i$; \textbf{break}
    \ENDIF
\ENDFOR

\STATE {\color{blue}\texttt{// Stage 4: Bonus token}}
\IF{all $\gamma$ tokens accepted}
    \STATE Append one additional token from $M_T$
\ENDIF

\UNTIL{EOS emitted \OR max length reached}

\RETURN $\mathbf{x}$

\end{algorithmic}
\end{algorithm}

\begin{table*}
\centering
\caption{Operational comparison between full-vocab KL and Top-K KL. Full-vocab KL avoids the irregular memory access introduced by Steps~1--2 and consists entirely of streaming, fused-kernel-friendly operations.}
\label{tab:topk_ops}
\scalebox{0.9}{
\begin{tabular}{lll}
\toprule
\textbf{Step} & \textbf{Full-vocab KL} & \textbf{Top-K KL} \\
\midrule
1: Top-K selection   & Not needed & $O(V)$, branch-intensive, irregular access \\
2: Gather by index   & Not needed & $O(K)$, indirect indexing, breaks coalescing \\
3: Log-softmax + KL  & $O(V)$, streaming, fusible & $O(K)$, streaming, fusible \\
\bottomrule
\end{tabular}}
\end{table*}

\subsection{Analysis of Top-K Vocabulary Truncation}
\label{appendix:topk_vocab}
\paragraph{Computational Comparison.}
A Top-K variant reduces the element-wise KL computation from $O(V)$ to $O(K)$, which is theoretically appealing.
However, it introduces two additional operations that full-vocab KL does not require: (1) a Top-K selection step involving branch-intensive, irregular memory access, and (2) a gather-by-index step that breaks coalesced memory patterns.
Table~\ref{tab:topk_ops} summarizes the operational differences.

The efficiency advantage of Top-K lies solely in Step~3, but realizing this advantage in practice requires fusing Steps~1--3 into a single kernel to mitigate the irregular memory access introduced by Steps~1 and~2.
With standard PyTorch implementations, these irregular operations negate the theoretical benefit.
We profile the wall-clock latency on Llama-3.1-8B-Instruct ($V\!=\!128{,}256$) for various $K$ values in Figure~\ref{ssy1122:topk_latency}. Notably, latency remains nearly constant across all $K$ values (0.125--0.130\,ms), confirming that the bottleneck lies in the Top-K selection and gather operations rather than the KL summation itself.
In contrast, full-vocab KL consists entirely of streaming element-wise operations that are naturally fused-kernel-friendly, achieving 1.8$\times$ lower latency than the fastest Top-K variant without any custom kernel optimization.

\begin{figure}[h]
\centering
\includegraphics[width=1.0\linewidth]{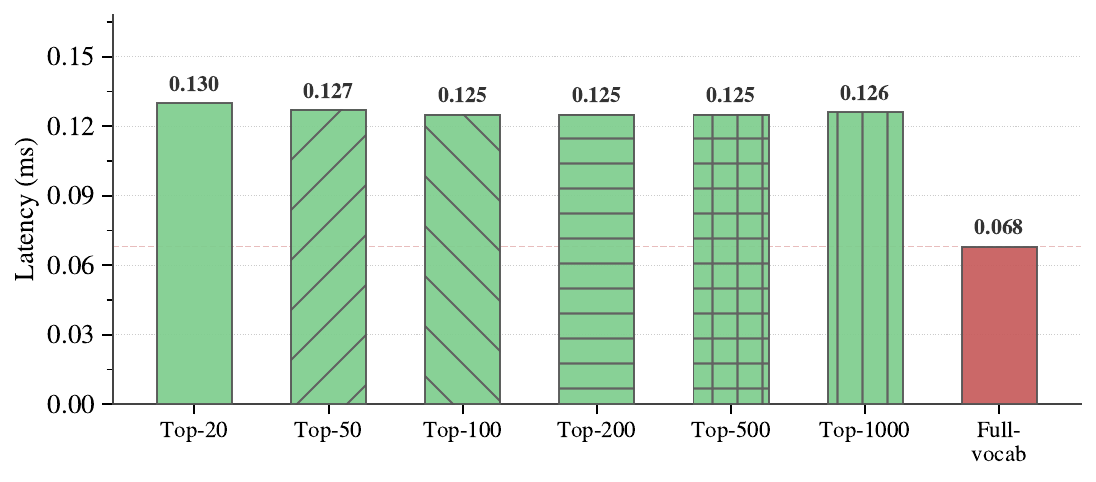}
\caption{Wall-clock latency (ms) of KL computation under different Top-K configurations. Full-vocab KL is 1.8$\times$ faster than the fastest Top-K variant due to its fully streaming memory access pattern.}
\label{ssy1122:topk_latency}
\end{figure}

\paragraph{Theoretical Compatibility}

We verify that the structural correspondence established in Theorem~\ref{ssy1122:structural_correspondence} is exactly preserved under Top-K truncation.
The core primitive $\Delta_{ij}(x) = (x_i - x_j)$ is defined on logit \emph{differences}, hence invariant to vocabulary restriction and renormalization.
Concretely, let $\tilde{P}_t$ and $\tilde{P}_d$ denote the renormalized target and draft distributions over a Top-K subset $S_K$.
Since Lemma~\ref{lem:a2_pairwise_decomp} is an algebraic identity valid for any probability simplex on any finite index set, Applying Corollary~\ref{cor:a2_kl_quadratic_aggregation} to $\tilde{P}_t$ and $\tilde{P}_d$ on $S_K$ yields $D_{\mathrm{KL}}(\tilde{P}_t \| \tilde{P}_d) \;\approx\; \sum_{i,j \in S_K} \tilde{\alpha}_{ij} \, \Delta_{ij}(x)^2$, with $\tilde{\alpha}_{ij} = \tfrac{1}{4}\,\tilde{P}_d(i)\,\tilde{P}_d(j).$ Combined with Proposition~\ref{prop:a5_classifier_primitive_linear}, the core claim (\emph{quadratic aggregation for KL vs.\ linear aggregation for the classifier, both over the same primitives}) holds exactly.
Truncation merely narrows the summation range from $V^2$ to $|S_K|^2$ without altering the algebraic structure of the correspondence.

\subsection{Analysis of the Top-1 Confidence Fallback}
\label{appendix:top1_fallback}
We provide a detailed analysis of the Top-1 confidence fallback mechanism,
examining both its activation frequency and its interaction with supervised verification methods.
All experiments are conducted on GSM8K using the Llama-3.2-1B-Instruct/Llama-3.1-8B-Instruct pair.

\paragraph{Activation Frequency.}
Table~\ref{tab:trigger-freq} reports the trigger statistics of the Top-1 confidence fallback
across different KL thresholds.
Even at the most aggressive threshold (i.e., $0.8$),
the fallback activates on fewer than 5\% of token-level verification decisions
(fewer than 0.4 times per query on average),
confirming that KL divergence serves as the dominant verification signal throughout decoding.

\begin{table}[h]
\centering
\small
\caption{Trigger frequency and performance of the Top-1 confidence fallback across KL thresholds.
``Trigger (\%)'' denotes the percentage of token-level decisions where the fallback is activated;
``Trigger / Query'' denotes the average number of activations per query.}
\scalebox{0.8}{
\begin{tabular}{llccccc}
\toprule[1.5pt]
& & Vanilla SP & 0.3 & 0.6 & 0.7 & 0.8 \\
\midrule
\multirow{4}{*}{\rotatebox{90}{\scriptsize w/ fallback}}
& Acc (\%) & 87.75 & 87.50 & 85.50 & 83.88 & 83.12 \\
& MAT            & 10.61 & 13.33 & 17.35 & 18.94 & 19.65 \\
& Trigger (\%)   & N/A   & 0.0   & 1.1   & 2.7   & 4.9   \\
& Trigger/Query & N/A  & 0     & 0.11  & 0.22  & 0.37  \\
\midrule
\multirow{2}{*}{\rotatebox{90}{\scriptsize w/o}}
& Acc (\%) & 87.75 & 87.60 & 86.62 & 85.25 & 84.38 \\
& MAT            & 10.61 & 13.45 & 17.21 & 18.75 & 19.80 \\
\bottomrule[1.5pt]
\end{tabular}}
\label{tab:trigger-freq}
\end{table}

We highlight three observations from Table~\ref{tab:trigger-freq}:

\begin{itemize}[itemsep=0pt]
\item \textbf{The fallback is triggered very infrequently.}
Over 95\% of all rejection events are governed by the KL-based criterion,
indicating that the confidence mask operates in a narrow regime.

\item \textbf{The fallback has negligible impact on overall performance.}
Both configurations (\textit{w/} and \textit{w/o}) yield comparable accuracy and MAT
across all thresholds. This is expected by design:
when the target confidence is high, the draft model overwhelmingly agrees with the target on the Top-1 token,
so KL divergence is naturally low and the draft would already be accepted.
The fallback only intervenes when both conditions, high target confidence and draft-target disagreement, co-occur, which is inherently rare.

\item \textbf{The fallback serves as a principled, zero-cost safeguard.}
It adds no learnable parameters, introduces negligible computational overhead
(a single scalar comparison), and provides a conservative fallback for potential edge cases.
Its design philosophy is analogous to a ``safety valve'': rarely activated under normal operation,
but available as an additional layer of robustness.
\end{itemize}

\paragraph{Interaction with Supervised Judges.}
To verify that the performance advantage of our method cannot be attributed to the Top-1 fallback alone,
we further combine it with AutoJudge and report results in Table~\ref{tab:autojudge-fallback}. As shown in Table~\ref{tab:autojudge-fallback}, incorporating the Top-1 fallback into AutoJudge produces negligible differences across all thresholds. The fallback targets a rare regime where high target confidence and draft-target disagreement co-occur, so its effect is negligible regardless of the underlying verification mechanism.
This confirms that the gains of our approach are attributable to the KL-based verification, which captures the full distributional divergence between draft and target in a training-free manner,
rather than compressing it into a binary accept/reject decision.

\begin{table}[h]
  \centering
  \caption{AutoJudge performance with and without Top-1 fallback under different KL thresholds.}
  \label{tab:autojudge-fallback}
  \scalebox{0.63}{
  \begin{tabular}{cc cccc c}
    \toprule[1.5pt]
    \textbf{Setting} & \textbf{Metric}
      & \textbf{Vanilla SP} & \textbf{0.003} & \textbf{0.008} & \textbf{0.02} & \textbf{0.03} \\
    \midrule
    \multirow{2}{*}{w/o fallback}
      & Acc & 87.75\% & 87.12\% & 85.12\% & 82.25\% & 80.38\% \\
      & MAT      & 10.61   & 14.06   & 16.18   & 17.26   & 18.51   \\
    \midrule
    \multirow{2}{*}{w/\ fallback}
      & Acc & 87.75\% & 86.60\% & 84.12\% & 82.25\% & 80.75\% \\
      & MAT      & 10.61   & 14.05   & 15.99   & 17.87   & 18.63   \\
    \bottomrule[1.5pt]
  \end{tabular}}
\end{table}

\subsection{Quality-Constrained Speedups}
\label{appendix:vllm_speedup}

We integrate our method into vLLM and evaluate quality-constrained speedups under various accuracy drop budgets $\delta$. The speedup is measured relative to vanilla speculative decoding (Vanilla SP), which itself already achieves roughly $2\times$ acceleration over standard autoregressive (AR) generation in vLLM. The absolute speedup over AR decoding is therefore considerably larger. Note that our current implementation is a straightforward integration without system-level optimization, and we believe there remains room for further engineering-level improvement. Results for two draft-target model pairs are presented in Table~\ref{tab:vllm_1b_8b} and Table~\ref{tab:vllm_8b_70b}, respectively. As shown, our method achieves up to $1.60\times$ speedup over vanilla speculative decoding with the Llama-3.2-1B / Llama-3.1-8B pair under a $\delta \leq 5\%$ accuracy drop budget, and up to $1.45\times$ with the Llama-3.1-8B / Llama-3.1-70B pair under $\delta \leq 3.0\%$.

\begin{table}[h]
\centering
\caption{Quality-constrained speedups on vLLM with Llama-3.2-1B-Instruct/Llama-3.1-8B-Instruct.}
\label{tab:vllm_1b_8b}
\scalebox{0.8}{
\begin{tabular}{lccc}
\toprule[1.5pt]
$\delta$ (Acc Drop) & Acc (\%) & Speed (tokens/s) & Speedup \\
\midrule
Vanilla SP        & 82.56 & 29.45 & 1.00$\times$ \\
$\leq$ 0.1\%      & 82.77 & 39.58 & 1.03$\times$ \\
$\leq$ 0.5\%      & 82.00 & 41.91 & 1.09$\times$ \\
$\leq$ 1.5\%      & 81.07 & 48.35 & 1.26$\times$ \\
$\leq$ 2.5\%      & 79.46 & 50.96 & 1.33$\times$ \\
$\leq$ 5\%        & 77.13 & 61.45 & 1.60$\times$ \\
\bottomrule[1.5pt]
\end{tabular}}
\end{table}

\begin{table}[h]
\centering
\caption{Quality-constrained speedups on vLLM with Llama-3.1-8B-Instruct/Llama-3.1-70B-Instruct.}
\label{tab:vllm_8b_70b}
\scalebox{0.8}{
\begin{tabular}{lccc}
\toprule[1.5pt]
$\delta$ (Acc Drop) & Acc (\%) & Speed (tokens/s) & Speedup \\
\midrule
Vanilla SP        & 93.00 & 24.19 & 1.00$\times$ \\
$\leq$ 0.5\%      & 92.69 & 27.21 & 1.12$\times$ \\
$\leq$ 1.0\%      & 92.02 & 31.77 & 1.31$\times$ \\
$\leq$ 1.5\%      & 91.41 & 32.78 & 1.36$\times$ \\
$\leq$ 2.5\%      & 90.36 & 34.29 & 1.42$\times$ \\
$\leq$ 3.0\%      & 90.27 & 35.10 & 1.45$\times$ \\
\bottomrule[1.5pt]
\end{tabular}}
\end{table}

\subsection{Discussion on Consensus vs. Correctness}
\label{appendix:consensus_vs_correctness}

One may argue that when both the draft and target models share the same erroneous reasoning, token-level KL consistency cannot lead to correct results. We clarify this concern below.

\begin{itemize}[itemsep=0pt]
\item \textbf{Speculative decoding targets distributional fidelity, not correctness.} The objective of speculative decoding is to approximate the target model's output distribution, not to surpass it. Low KL divergence signals that the draft faithfully reproduces the target's behavior. Accepting consensus tokens is therefore the \textit{intended behavior}.

\item \textbf{Trained judges are equally unable to detect shared errors.} Supervision for critical token detection, whether mined automatically (as in AutoJudge \cite{DBLP:conf/nips/GaripovVelikonivtsev25}) or annotated manually~\cite{DBLP:conf/iclr/BachmannAnagnostidis25}, is derived exclusively from draft-target divergence points. When both models produce identical output, no divergence is detected and no token is flagged as critical. This limitation is inherent to the draft-then-verify paradigm itself.

\item \textbf{Correcting the target model's own errors is an orthogonal direction.} Addressing such errors requires fundamentally different techniques (e.g., tool-augmented reasoning or reward models), which fall outside the scope of speculative decoding.
\end{itemize}
\end{document}